\documentclass[a4paper,fleqn]{cas-dc}

\usepackage[numbers]{natbib}
\usepackage{url}
\usepackage{caption}
\usepackage{subcaption}

\usepackage{threeparttable}
\usepackage[ruled]{algorithm2e}
\usepackage{amsthm}

\usepackage[figuresright]{rotating}

\newtheorem{proposition}{Proposition}
\makeatletter
\newcommand\figcaption{\def\@captype{figure}\caption}
\newcommand\tabcaption{\def\@captype{table}\caption}

\makeatother
\def\tsc#1{\csdef{#1}{\textsc{\lowercase{#1}}\xspace}}
\tsc{WGM}
\tsc{QE}

\begin{document}
\let\WriteBookmarks\relax
\def\floatpagepagefraction{1}
\def\textpagefraction{.001}

\shorttitle{Dynamic Data-Free Knowledge Distillation by Easy-to-Hard Learning Strategy}    


\title {\LARGE Dynamic Data-Free Knowledge Distillation by Easy-to-Hard Learning Strategy}  

\tnotemark[1] 


%

\author[1]{Jingru Li}

\cormark[1]


\ead{jrlees@zju.edu.cn}


\credit{Conceptualization of this study, Methodology, Software, Writing - Original Draft, Validation, Formal analysis, Investigation, Writing - Review \& Editing, Visualization}

\affiliation[a]{organization={Zhejiang University},
            addressline={Zheda Rd.}, 
            city={Hangzhou},
            postcode={}, 
            state={Zhejiang},
            country={China}}




\author[1]{Sheng Zhou}

\fnmark[2]

\ead{zhousheng_zju@zju.edu.cn}

\credit{Methodology, Writing - Review \& Editing, Supervision}




\author[1]{Liangcheng Li}

\fnmark[3]

\ead{liangcheng_li@zju.edu.cn}

\credit{Resources, Writing - Review \& Editing, Supervision, Project administration}

\author[1]{Haishuai Wang}

\fnmark[4]

\ead{haishuai.wang@zju.edu.cn}

\credit{Methodology, Writing - Review \& Editing, Supervision}

\author[1]{Jiajun Bu}

\fnmark[5]

\ead{bjj@zju.edu.cn}

\credit{Funding acquisition, Writing - Review \& Editing, Supervision}

\author[1]{Zhi Yu}

\fnmark[6]

\ead{yuzhirenzhe@zju.edu.cn}


\credit{Funding acquisition, Writing - Review \& Editing, Supervision, Visualization}


\cortext[1]{Zhi Yu}

\fntext[1]{Corresponding author of this paper is Zhi Yu.}


\begin{abstract}
Data-free knowledge distillation (DFKD) is a widely-used strategy for Knowledge Distillation (KD) whose training data is not available. It trains a lightweight student model with the aid of a large pretrained teacher model without any access to training data. However, existing DFKD methods suffer from inadequate and unstable training process, as they do not adjust the generation target dynamically based on the status of the student model during learning. To address this limitation, we propose a novel DFKD method called CuDFKD. It teaches students by a dynamic strategy that gradually generates easy-to-hard pseudo samples, mirroring how humans learn. Besides, CuDFKD adapts the generation target dynamically according to the status of student model. Moreover, We provide a theoretical analysis of the majorization minimization (MM) algorithm and explain the convergence of CuDFKD. To measure the robustness and fidelity of DFKD methods, we propose two more metrics, and experiments shows CuDFKD has comparable performance to state-of-the-art (SOTA) DFKD methods on all datasets. Experiments also present that our CuDFKD has the fastest convergence and best robustness over other SOTA DFKD methods.
\end{abstract}




\maketitle

\section{Introduction} \label{sec: intro}
Knowledge Distillation (KD) is a popular strategy to train smaller student models by leveraging knowledge from large, pretrained teacher models, and it significantly improves the performance of small student models \cite{hinton2015distilling,zhao2022decoupled}. Due to privacy concerns \cite{truong2020data} or resource constraints for training on large benchmarks like ImageNet, the training data is not accessible. To address this challenge, Data-Free Knowledge Distillation (DFKD) or Zero-Shot Knowledge Distillation (ZSKD) \cite{wang2021zero} has been proposed. It allows the training of student models without real training data, by using the teacher model's weights to reverse inference knowledge and generating pseudo samples to train the student model. This knowledge, which represents the pretrained weight of the teacher model, encapsulates prior information about the original training data and provides the target for the generation of the pseudo samples.

In Figure \ref{fig:prior-based methods} to \ref{fig:cmi}, different targets for the generation are set in different DFKD methods to reverse the above knowledge. Such targets include data priors \cite{chen2019data, yin2020dreaming, Nguyen2022BlackboxFK} (Figure \ref{fig:prior-based methods}), adversarial samples \cite{choi2021qimera, fang2019data} (Figure \ref{fig:boundary-based methods}), and previous memory samples \cite{fang2021contrastive, binici2022robust, binici2022preventing} (Figure \ref{fig:cmi}). They have achieved remarkable success, whose trained student models are comparable in performance to those in data-driven KD. Such DFKD methods prescribe a fixed target during the process of DFKD, and finally optimize a generative model and a student model. 

However, these fixed target settings of DFKD may be more challenging for the student at early learning process. We believe that a more productive method should adjust the training and generation targets with the status of the student model. This approach emulates human learning, where teachers modify their teaching strategies dynamically based on the student's learning ability. For instance, humans use a dynamic procedure to learn easy knowledge like arithmetic before hard knowledge like calculus, while a static approach is employed to learn both simultaneously as aforementioned DFKD methods do. Obviously, the former learning strategy is better for the student of the better learning experience and effeciency. Back to the field of machine learning, such dynamic target strategy is consistent with curriculum learning(CL) techniques\cite{bengio2009curriculum,wang2021survey,soviany2022curriculum}. It first learns the easiest part of the training data, followed by more difficult part, which is a \textbf{dynamic} target setting. It promises faster convergence rate and better local minima for the machine learning problem.

\begin{figure*}[!htbp]
     \centering
     \begin{subfigure}[b]{0.235\textwidth}
         \centering
         \includegraphics[width=\textwidth]{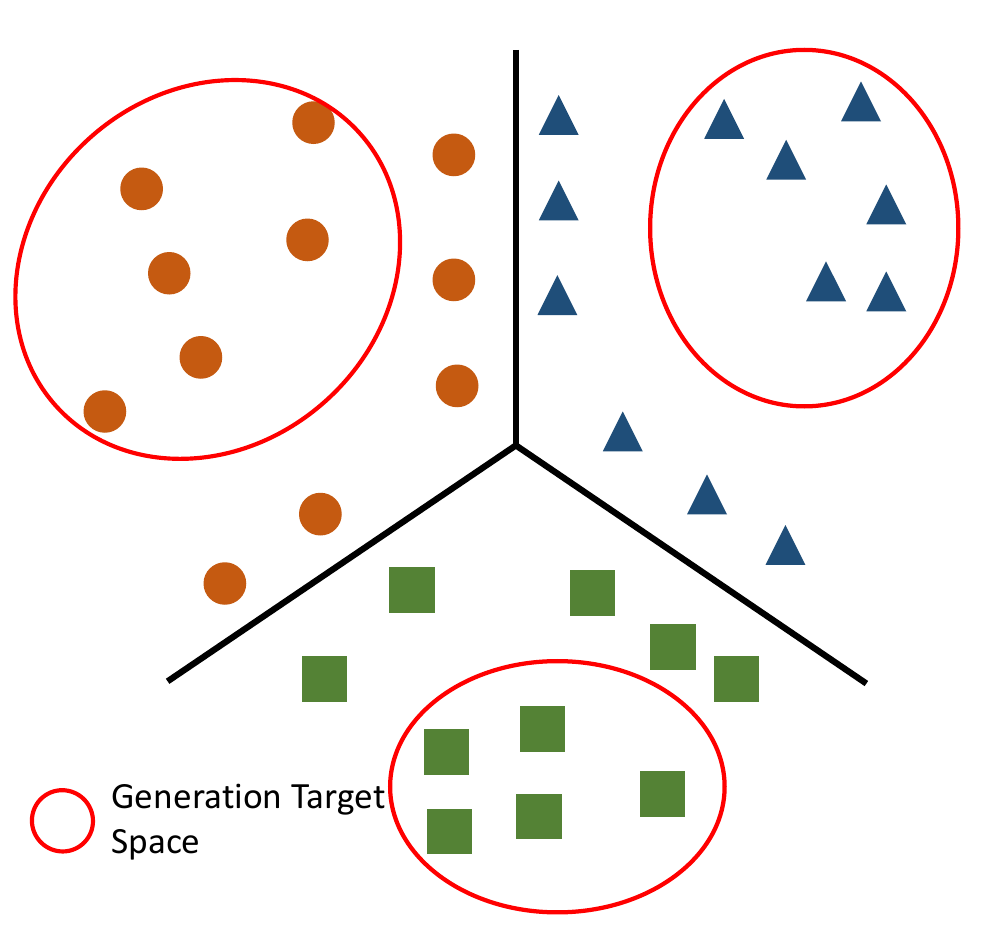}
         \caption{Prior-based methods.}
         \label{fig:prior-based methods}
     \end{subfigure}
     \hfill
     \begin{subfigure}[b]{0.235\textwidth}
         \centering
         \includegraphics[width=\textwidth]{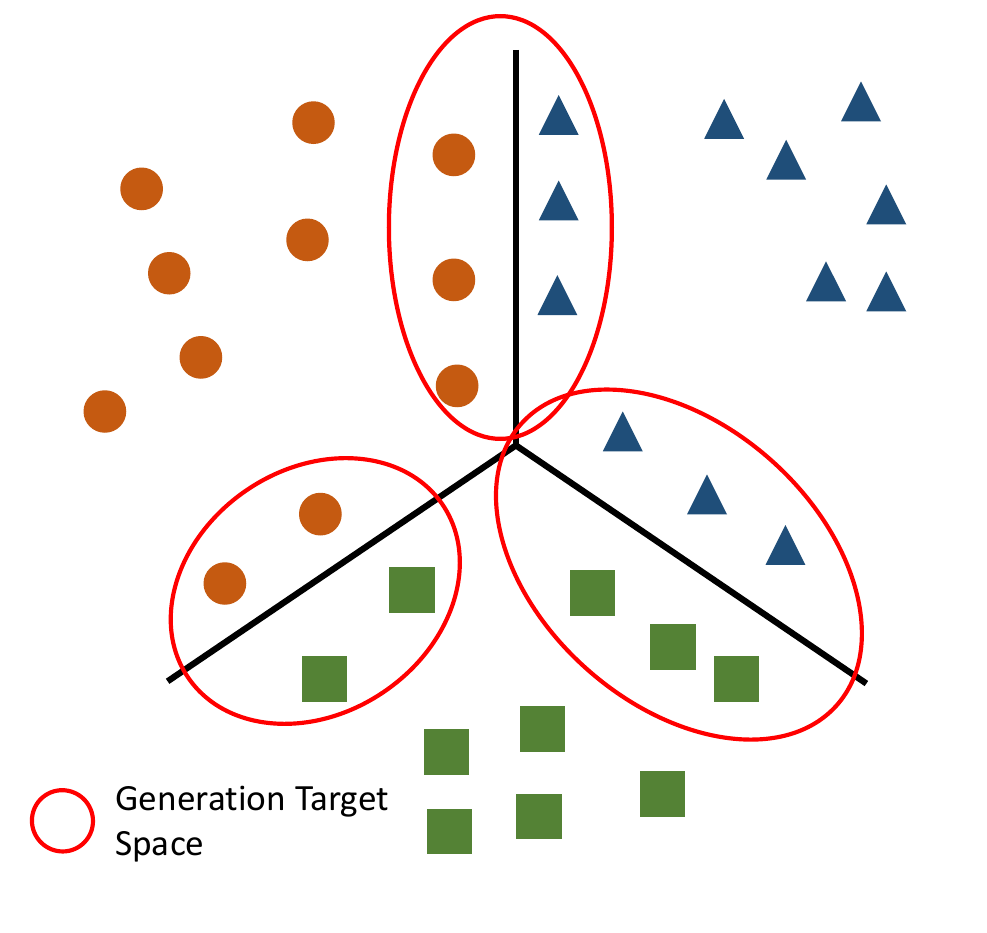}
         \caption{Boundary-based methods.}
         \label{fig:boundary-based methods}
     \end{subfigure}
     \hfill
     \begin{subfigure}[b]{0.235\textwidth}
         \centering
         \includegraphics[width=\textwidth]{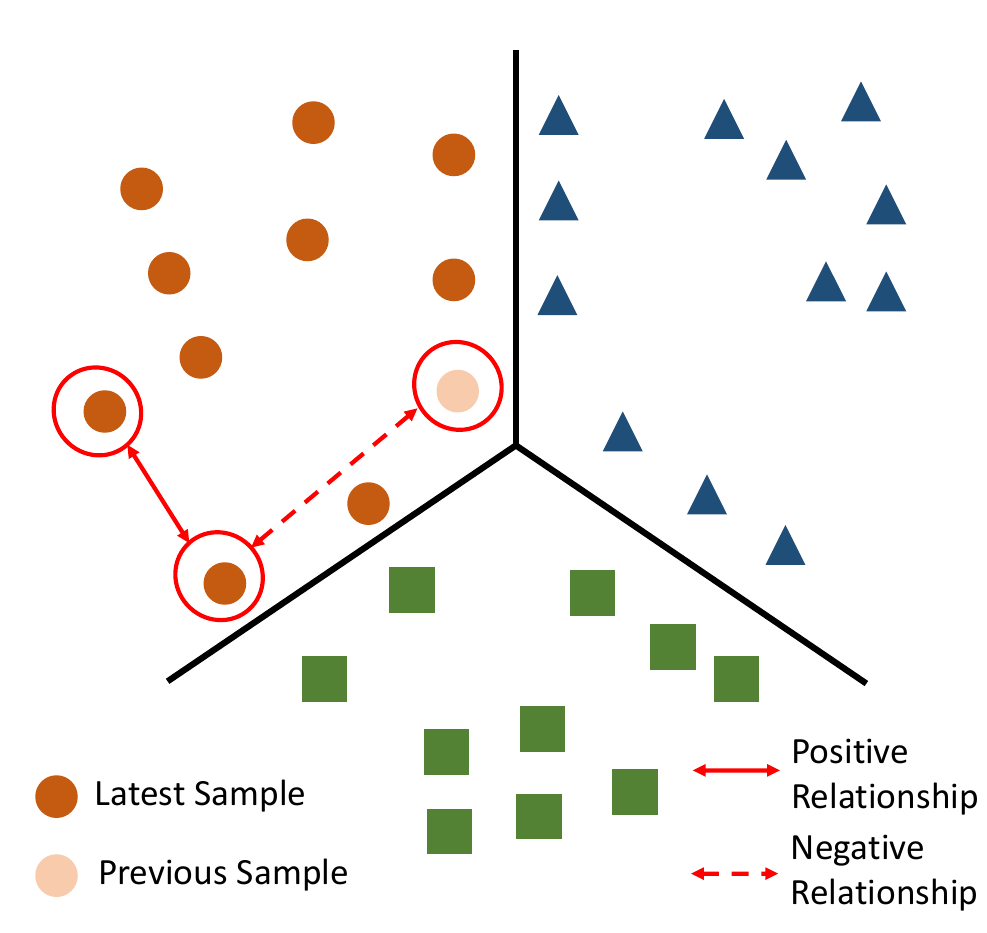}
         \caption{CMI\cite{fang2021contrastive}.}
         \label{fig:cmi}
     \end{subfigure}
     \hfill
     \begin{subfigure}[b]{0.235\textwidth}
         \centering
         \includegraphics[width=\textwidth]{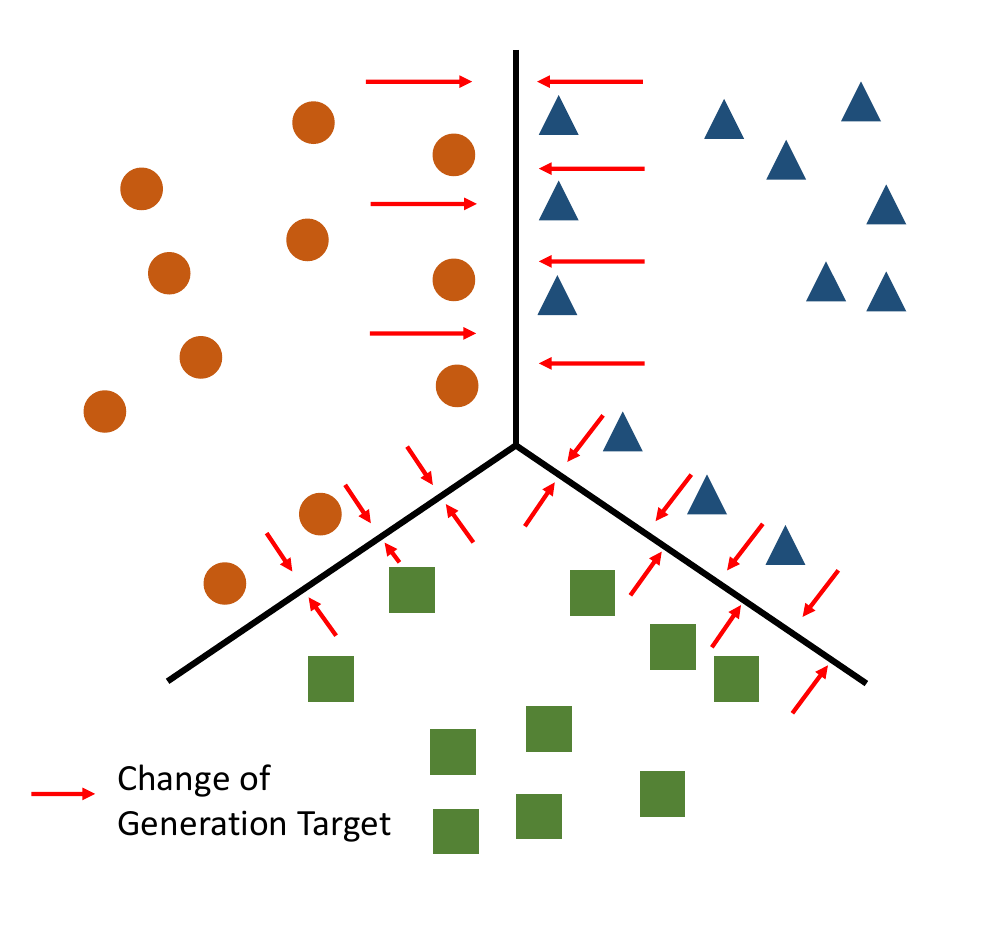}
         \caption{CuDFKD(Ours).}
         \label{fig:cudfkd}
     \end{subfigure}
        \caption{Differences between CuDFKD and other methods. Black lines represent the decision boundary by the teacher model, and the dots with different shapes represent pseudo samples with different classes. a) prior-based methods learn the generation target at the space of data prior. b) boundary-based methods learn pseudo samples close to the decision boundary among classes. c) memory-based methods like CMI\cite{fang2021contrastive} use the previously generated data samples as the generation target. d) Our CuDFKD learns a dynamic generation target from data prior to the decision boundary with time. Better viewed in color.}
        \label{fig:figure1}
\end{figure*}

To propose a dynamic DFKD method, we present a brand new dynamic learning approach for DFKD based on Curriculum Learning in this paper, called \textbf{Cu}rriculum \textbf{DFKD} (CuDFKD). We introduce a dynamic and adaptive strategy for generating pseudo samples in CuDFKD, based on an easy-to-hard approach. Our method allows the learning process to adapt to the student's capability. The difficulty for different training timestamps in DFKD is defined as the divergence between the teacher and student models. We generate pseudo samples from data prior to decision boundaries, which serves as a dynamic signal indicating the level of difficulty. Figure \ref{fig:cudfkd} illustrates the main idea of CuDFKD. To update the student model and generator parameters, we propose an alternative updating strategy (AOS) that is simple yet effective. Theoretical support for CuDFKD comes from self-paced learning (SPL) \cite{wang2021survey}. In practice, the dynamic modules are implemented by adjusting the gradient of adversarial samples and reweighting the generated pseudo samples in the objective function. Experimental results demonstrate that CuDFKD performs comparably to other state-of-the-art (SOTA) DFKD methods, with the added advantage of superior convergence.

The contributions of this paper are summarized as follows:

\begin{itemize}

  \item We propose a simple but effective DFKD method by dynamically adjusting the difficulty of generated pseudo samples, called CuDFKD. It's the first method to dynamically match the target generation with the status of student model.
  
  \item We demonstrate the convergence of CuDFKD by the theory of MM and VC-dimension. We also provide a detailed analysis and discussion of the advantages of CuDFKD over memory-based and boundary-based DFKD methods.

  \item We first utilize two more metrics to measure the performance of DFKD. Besides, Our CuDFKD performs best in CIFAR10, CIFAR100, Tiny ImageNet over various teacher-student pairs and metrics.
  
  \item The designed dynamic module in CuDFKD is proved to be effective by the best performance for both KD training and pseudo sample generation. Notably, our proposed CuDFKD method has the fastest convergence over other DFKD methods. It preserves the best performance even in the presence of noisy teachers.
\end{itemize}

This paper is organized as follows: section \ref{sec: rw} reviews some related work about DFKD and curriculum learning, and section \ref{sec: preliminary} discusses some basic concepts about curriculum learning and self-paced learning. Section \ref{sec: theo} outlines how to build up CuDFKD and proves its convergence by MM theory. Section \ref{sec: exp} reports some results and discussion about CuDFKD. Finally, section \ref{sec: future} concludes this paper and discusses future work. 

\section{Related Work}\label{sec: rw}

This section reviews the previous work of DFKD and curriculum learning and compares their differences with CuDFKD.

\subsection{Data-Free Knowledge Distillation}\label{subsec: dfkd_rl}

Data-Free Knowledge Distillation (DFKD) generates pseudo samples to optimize student model. According to the kind of target during generation pseudo samples, we categorize DFKD methods into three types, i.e.,  prior-based DFKD, boundary-based DFKD, and memory-based DFKD.

Prior-based DFKD methods \cite{chen2019data,yin2020dreaming,nayak2019zero,wang2021data,Nguyen2022BlackboxFK} imitate the original data distribution using inversion methods. ZSKD \cite{nayak2019zero,wang2021zero,micaelli2019zero} and SoftTarget \cite{wang2021data} model the output label distribution or intermediate feature maps with simple distributions. DAFL \cite{chen2019data}, DeepInversion, and Adaptive DeepInversion \cite{yin2020dreaming} use different regularization terms, such as adversarial divergence or batch normalization statistics, to generate more realistic and useful pseudo data samples. Additionally, some limited-data KD models use similar targets to guide the training stage, such as \cite{fang2021mosaicking}.

Boundary-based DFKD methods \cite{fang2019data,yin2020dreaming,choi2021qimera,choi2020data,truong2021data} are motivated by the thought of adversarial learning, where generative models are optimized to maximize the gap between the output distribution of teacher and student models. \cite{truong2021data,zhang2022qekd} treat the pretrained teacher model as a black-box model and use a min-max game to update the generator and student models. DFQ\cite{choi2020data} dynamically balances the sample generation between prior-based and boundary-based approaches and incorporates instance/category entropy loss for prior regularization. Qimera\cite{choi2021qimera} explores the effect of boundary samples on model quantization, providing a similar perspective to KD. They learn samples near the teacher model's decision boundary, treating them as hard samples. Such methods generally produce robust student models and high performance.

During the training process, \cite{binici2022preventing} points out that catastrophic forgetting may occurs, whereby some learned knowledge or gradients could not be preserved, leading to the student model getting trapped in local minima. Memory-based DFKD methods are proposed to rephrase the learned knowledge from the early student model. CMI\cite{fang2021contrastive}, PRE-DFKD\cite{binici2022robust}, and MB-DFKD\cite{binici2022preventing} are motivated by continual learning\cite{mazur2022target}. They use memory samples to avoid the catastrophic forgetting during KD training, achieving outstanding performance on different benchmarks. They preserve the knowledge from the early student model by setting up a memory bank\cite{fang2021contrastive,binici2022preventing} or reconstruction \cite{binici2022preventing} process, thereby assisting the learning of student model. CuDFKD, on the other hand, only uses a non-parameterized training scheduler and difficulty measurer, eliminating extra memory usage or parameter training.

DFKD can also be extended to graph data as "graph-free distillation"\cite{deng2021graph,wang2021online}, which is applicable in tasks such object detection \cite{banitalebi2021knowledge} and image super-resolution \cite{zhang2021data}. Additionally, some work\cite{Fang2022UpT1} proposes faster DFKD methods by meta-learning.

\subsection{Curriculum Learning}\label{subsec: rl in cl}
Curriculum learning (CL) has been widely used to train models by an easy-to-hard strategy \cite{bengio2009curriculum, wang2021survey,soviany2022curriculum}. Bengio \cite{bengio2009curriculum} provides a clear illustration of its convergence of it.

Automatic CL is a dynamic strategy to adjust the difficulty by the feedback of training process. Self-paced learning (SPL) \cite{kumar2010self} is the most widely used automatic CL method assigning data with different difficulties based on the training losses at each timestamp (or epoch). Several theoretical studies \cite{ma2018convergence,meng2017theoretical} provide a deep understanding of SPL and are categorized into majorization minimization (MM) \cite{caflisch1998monte} algorithm and concave optimization. They use transfer learning \cite{soviany2022curriculum} and uncertainty \cite{zhou2020uncertainty} as specific representative techniques to transfer from teacher. Recently, SPL is also applied to the field of unsupervised learning\cite{li2022unsupervised}, clustering\cite{zhang2022spaks}, anomaly detection\cite{yu2022deep} and graph\cite{gong2022self}.

Besides, some data-driven KD methods also use CL to enhance student learning from the teacher. For example, Xiang et al. \cite{xiang2020learning} use SPL for instance selection in long-tailed datasets, and Li et al. \cite{li2021dynamic} use a similar uncertainty curriculum to distill models from large pretrained language models. In this work, CL provides an adaptive training target for the generation process of DFKD. Additionally, it provides a theoretical understanding of accelerating the convergence of DFKD methods and contributes to the usage of CL in knowledge distillation.

\begin{figure*}[!htbp]
    \centering
    \includegraphics[width=0.95\textwidth]{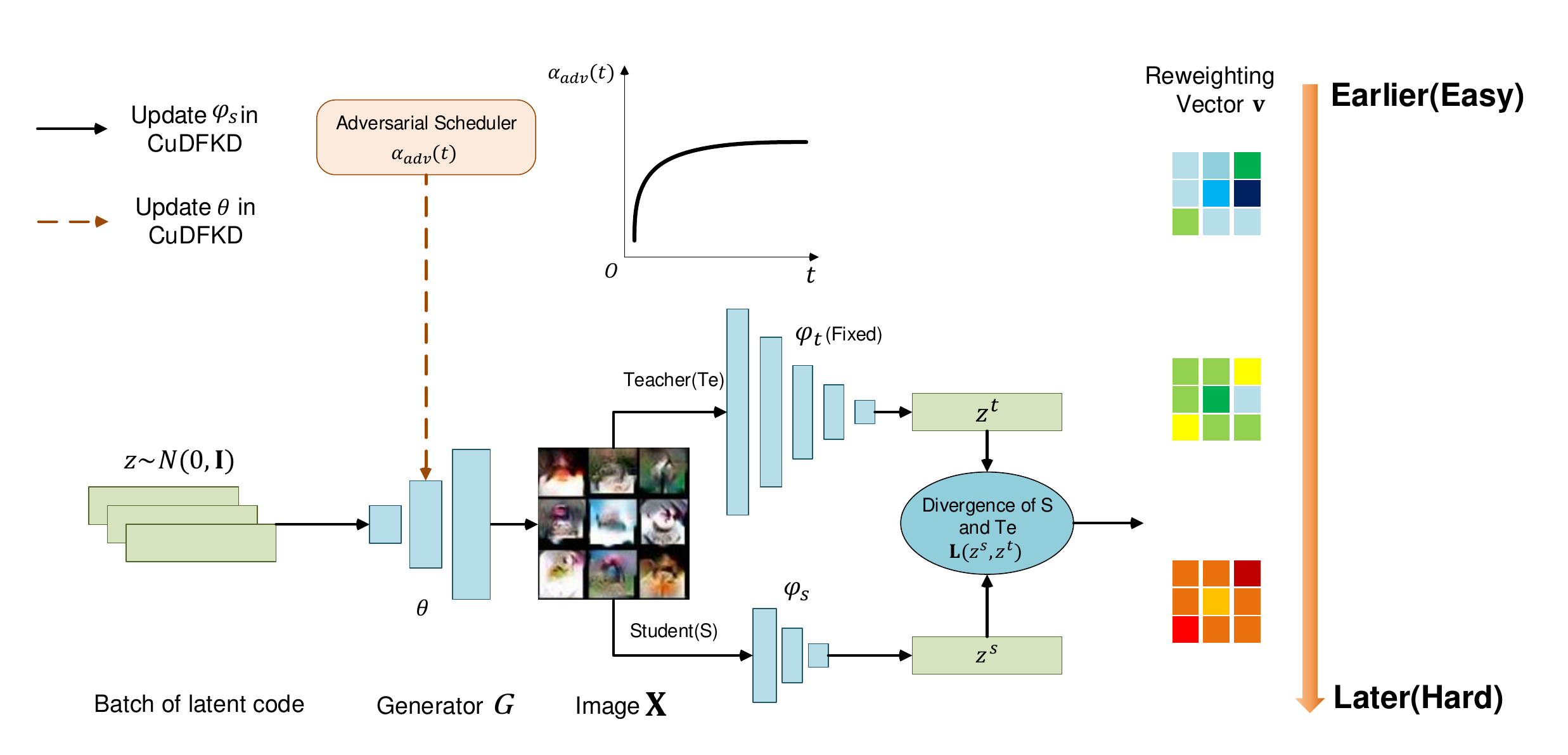}
    \caption{The mainstream of CuDFKD. We construct a dynamic DFKD generation and training target by 1) an adversarial scheduler $\alpha_{adv}(\tau)$ to adjust the gradient of the generator $G$, and 2) Reweighting factor $\mathbf{v}$ to reweight the generated pseudo samples. For $\mathbf{v}$, the warmer the colors are, the higher value of $\mathbf{v}$ is. Better viewed in screen.}
    \label{fig:workflow}
\end{figure*}

\section{Preliminary on CL} \label{sec: preliminary}
Curriculum Learning(CL) is widely used in the data-driven machine learning, i.e., given a training dataset $\mathcal{D} = \{(\mathbf{x}, y)\}$, where $y$ is the label of specific data $\mathbf{x}$. CL defines a difficulty scheduler $d(\mathbf{x}, t)$ to evaluate the difficulty of data samples and a training scheduler to find the subset $\mathbb{B}_t \in \mathcal{D}$ of training data at this difficulty. CL is designed to get the optimal parameter $\mathbf{w}^*$ by some schedulers, and the data subset series is organized to be easy-to-hard. If the difficulty scheduler function is continuous, the subset series should satisfy $\partial d(\mathbb{B}_i, t) / \partial t \geq 0$.

Instead of manually setting two schedulers, self-paced learning (SPL)\cite{kumar2010self} dynamically update difficulty by the training loss $\mathbf{L}(\mathbf{x}, y; \Theta)$. What's more, they define a reweighting factor $\mathbf{v}$ as a training scheduler, i.e.,

\begin{equation}\label{equ: spl}
    \min_{\mathbf{w} \in \mathbb{R}^d, \mathbf{v} \in [0,1]^d} \mathbf{F}(\mathbf{w}, \mathbf{v}) = 
    \mathbf{v}(\lambda, \mathbf{L})^T\mathbf{L}(\mathbf{x}, y; \Theta) + g(\lambda, \mathbf{v}),
\end{equation}

where the convex function $g(\lambda, \mathbf{v})$ is called a Self-Paced regularizer(SP-regularizer). It's an alternative optimization problem for parameter $\mathbf{w}$ and reweighting factor $\mathbf{v}$. Here $g(\lambda, \mathbf{v})$ is designed under some constraints. It is can be explicit defined by primary functions\cite{wang2021survey,soviany2022curriculum}. 

\section{Design and Discussion of CuDFKD}\label{sec: theo}

In this section, we implement the SPL framework to DFKD and propose a dynamic DFKD algorithm to update all the parameters, called \textbf{CuDFKD}. To better design the dynamic module of CuDFKD, there still exist some gaps, including:

\begin{itemize}
    \item \textbf{Q1:} How to define the loss function of the generative model $\mathbf{L}_{\theta}$ to set dynamic generation targets?
    \item \textbf{Q2:} How to design the loss function for KD learning to make the DFKD more adaptive to the student model?
\end{itemize}

In this work, we present an alternative updating strategy (AOS) that aims to update all parameters of CuDFKD more effectively. Specifically, we alternately train the generator $\theta$ and student model $\phi_s$ in each epoch. Figure \ref{fig:workflow} provides a visual representation of the overall workflow of our approach.

\subsection{Dynamic Generation Module} \label{sec: method}

In data-driven KD, the teacher plays a crucial role in ensuring a well student \cite{wang2021knowledge,bang2021distilling}, and this becomes even more critical in DFKD. For \textbf{Q1}, when updating $\theta$, we construct a dynamic generation module, with an objective function for generating pseudo samples motivated by previous DFKD methods \cite{choi2020data,fang2021contrastive,yin2020dreaming}, namely, 

\begin{equation}\label{equ: dfkd}
    \min_{\theta} \mathbf{L}_{\theta} =  \mathbb{E}_{x \sim p_{\theta}} [-\alpha_{adv}(\tau) \mathcal{L}_{adv}(x) + \mathcal{L}_{\theta}(x) ],
\end{equation}

Here $\mathcal{L}_{adv}(x) = D(f_{\phi_t}(x), f_{\phi_s}(x))$, and divergence function $D$ is either Jensen–Shannon divergence (JS divergence) or Kullback–Leibler divergence (KL divergenece). $\mathcal{L}_{\theta}(x)$ is a loss function for regularizing the pretrained teacher model and improving the quality for the generation of pseudo samples. $\mathcal{L}_{\theta}(x)$ includes categorical entropy\cite{choi2020data}, and batch normalization statistics alignment\cite{yin2020dreaming}. The regularization item $L_{\theta}$ is designed by,

\begin{equation}\label{equ: l_theta}
\begin{aligned}
    \mathcal{L}_{\theta}(x) &= \alpha_{bn} \mathcal{L}_{bn} + \alpha_{oh} \mathcal{L}_{oh} \\
    &= \alpha_{bn} \sum_{l} (\mathcal{D}(\mu_{(l)}, \mu_{bn,(l)}) + \mathcal{D}(\sigma_{(l)}^{2}, \sigma_{bn, (l)}^2)) \\ &+ \alpha_{oh} \sum_{i=1}^C CE(\mathbf{z}^t_{i}, \arg \max_j {z}^t_{i,j}),
\end{aligned}
\end{equation}

The objective of Equ. \ref{equ: l_theta} is two-fold. The first term aims to align the intermediate BN layer statistics of the pseudo sample with those of real data, thereby ensuring statistical consistency between the two distributions. Besides, the second term $CE(\mathbf{z}^t_{i}, \arg \max_j {z}^t_{i,j})$ serves as a cross-entropy loss between the teacher output logit $\mathbf{z}^t_{i}$ and the pseudo label $y_i = \arg \max_j {z}^t_{i,j}$, which balances the class distribution in the output space of the teacher model. This objective, represented by $\mathcal{L}_{\theta}$, facilitates the alignment of the generated distribution with the original data distribution $p_{data}$.

Compared with previous DFKD methods\cite{yin2020dreaming,choi2020data,fang2021contrastive}, our approach adjusts the difficulty of generated samples by tuning the hyperparameter $\alpha_{adv}$. Specifically, by setting $\alpha_{adv}(\tau) = 0$, we utilize only the data-prior term $\mathcal{L}_{\theta}(x)$ in Equ. \ref{equ: dfkd} as the generation target, while setting $\alpha_{adv}(\tau) = \infty$ results in the sole use of the adversarial term in Equ. \ref{equ: dfkd}. In CuDFKD, we design the generation target as shown in Equ. \ref{equ: grad_adv} to dynamically balance the data-prior and decision boundary. To ensure non-decreasing difficulty during training, we employ a predefined function $\alpha_{adv}(\tau)$, whose detailed design are discussed in Section \ref{subsec: hyper}. The adjustment of the difficulty gradient is illustrated as orange dashed arrows in Figure \ref{fig:workflow}, thereby ensuring a dynamic generation target.

\subsection{Dynamic Training Module}

In this section, we design a dynamic training module of DFKD, specifically in relation to \textbf{Q2}. Therefore, we adjust the difficulty when updating $\mathbf{\phi}_s$ in DFKD. As discussed in section \ref{subsec: rl in cl}, automatic curriculum learning (CL) is a technique that dynamically adjusts the difficulty of the training process, thereby enhancing both training speed and convergence. By implementing CL through appropriate pacing functions or transferring from the teacher, we can further improve the performance of DFKD. To achieve this, the pretrained teacher model serves as a better scheduler for the guidance of DFKD. This approach improves the effectiveness of DFKD by ensuring that the difficulty of DFKD process is adapted in a dynamic and efficient manner.

First of all, we construct a difficulty measurer for the above dynamic information. Compared with previous CL methods\cite{bengio2009curriculum}, the dataset $\mathcal{D}$ is not available. Therefore, we build a difficulty measurer $d(p_{\theta}(x), \tau)$ with respect to the built distribution estimator $p_{\theta}(x)$ and timestamp $\tau$. It should satisfy the following criteria:

\begin{itemize}
    \item It's monotonically non-decreasing w.r.t. timestamp $\tau$, i.e., $\frac{\partial d(p_{\theta}(x), \tau)}{\partial \tau} \geq 0$, which promises a easy-to-hard learning strategy.
    \item It guides the training of DFKD, i.e., to optimize a proper parameter by $\theta^* = \arg \min_{\theta} f(d(p_{\theta}(x), \tau), \mathbf{v}^*)$, where $f$ is a type of loss function.
\end{itemize}

Inspired by boundary-based methods\cite{choi2021qimera, choi2020data}, we consider pseudo samples close to the decision boundary as hard samples. Therefore, we define a tractable difficulty scheduler for a given dataset $\mathcal{D}$ as $d(\mathcal{D}, t) = \mathbb{E}_{x \sim \mathcal{D}} [D(p_{\phi_t}(x) || p_{\phi_s}(x))]$, where $\phi_t$ and $\phi_s$ denote the parameters of the teacher and student models, respectively. Here $D$ is used to minimize the objective function $\mathbf{F}(\mathbf{w}, \mathbf{v})$ in the DFKD method. In the setting of DFKD, the dataset $\mathcal{D}$ represents the distribution of the trained dynamic generation module. The dynamic strategy $\mathbf{v}^*$ is then calculated to control the level of difficulty at a fine-grained level, and it is considered as another dynamic target for the training stage of DFKD. To this end, we define $\lambda(\tau)$ as a linearly increasing function with respect to the epoch $\tau$. The objective function for optimizing $\mathbf{\phi_s}$ can then be expressed as follows:

\begin{equation}\label{equ: s}
\begin{aligned}
    &\min_{\mathbf{\phi}_s} \mathbf{F}(\mathbf{\phi}_s) \\ &= \mathbf{v}^*(\lambda, \mathbf{L})^\text{T} \mathbb{E}_{x \sim p_{\theta}} [D(\frac{z_t(x)}{T}, \frac{z_s(x; \mathbf{\phi}_s)}{T})] + g(\lambda, \mathbf{v}^*) \\ &=
    \sum_{i=1}^N v_i^*(\lambda, l_i ) l_{i, \phi_s}(G_{\theta}(z)) + g(\lambda, v^*_i),
\end{aligned}
\end{equation}

The logit output at the final layer of the teacher and student models is denoted by $z_t$ and $z_s$, respectively. The temperature, $T$, defined in the original KD method \cite{hinton2015distilling} is used to soften the output distribution of $\mathbf{\phi}_t$ and $\mathbf{\phi}_s$. The loss function $\mathbf{L}$ in Equ. \ref{equ: s} is given by $D(\frac{z_t(x)}{T}, \frac{z_s(x; \mathbf{\phi}_s)}{T})$, where $D$ can be one of KL divergence, JS divergence, or $l_1$-loss \cite{choi2020data,wang2021data,fang2021contrastive}. The second line of Equ. \ref{equ: s} represents the implementation of the first line. Therefore, to update $\phi_s$, we compute the dynamic difficulty by $\mathbf{v}^*(\lambda, \mathbf{L})^T \mathbb{E}_{x \sim p_{\theta}} [D(\frac{z_t(x)}{T}, \frac{z_s(x; \mathbf{\phi}_s)}{T})]$.

\begin{algorithm}
    \caption{CuDFKD}
    \label{alg: main}
    \LinesNumbered 
    \KwIn{A pretrained teacher model $\phi_t$, Student model $\phi_s$, $\lambda$-scheduler $\lambda(\tau)$, steps for training $\theta$, $n_G$, steps for training $\mathbf{\phi}_s$, $n_s$}
    \KwOut{A trained student model $\phi_s$.}
    
    $\tau = 0$;
    
    \While{not converge}{

    \For{i = 1:$n_G$}{
    
        $\alpha_{adv} \leftarrow \alpha_{adv}(\tau)$ by Equ. \ref{equ: grad_adv};
    
        Calculate $\mathbf{L}_{\theta}$ by Equ. \ref{equ: dfkd};
        
        $\mathbf{L}_{\theta}$.backward();
    }
    
    \For{t = 1:$n_s$}{
    
        $\lambda \leftarrow \lambda(\tau)$;
        $\mathbf{v} \leftarrow \mathbf{v}^*(\lambda, \mathbf{L})$;
    
        Calculate $\mathbf{F}(\mathbf{\phi}_s) $ by Equ. \ref{equ: s};
        
        $\mathbf{F}(\mathbf{\phi}_s)$.backward();
        
    }

    }
    \end{algorithm}

\subsection{Algorithm of CuDFKD}

This section proposes the alternative updating strategy (AOS) algorithm for CuDFKD. It incorporates self-paced learning with various strategies to make a dynamic learning process. The algorithm is presented in Algorithm \ref{alg: main}, where $\theta$ and $\phi_s$ are updated alternatively using equations \ref{equ: l_theta} and \ref{equ: s}. Specifically, the generation dynamic module and training dynamic module is represented by lines 4 and 9, respectively.

As discussed in Section \ref{sec: preliminary}, Algorithm \ref{alg: main} implements an alternative updating strategy (AOS) for three parameters: the generator $\theta$ (lines 5-6), student model $\phi_s$ (lines 10-11), and reweighting factor $\mathbf{v}$ (line 9). Assuming that these parameters have time costs $T_g$, $T_s$, and $T_v$, respectively, the time cost of an inner loop (lines 3-12) is given by $T_{inner} = O(n_GT_g + n_s(T_s + T_v))$. Furthermore, given that the hyperparameter $\lambda(\tau)$ and reweighting factor $\mathbf{v}$ are predefined primary functions, we assume that $T_g$ and $T_s$ are significantly larger than $T_v$. Therefore, we can approximate the time cost of an inner loop as $T_{inner} \approx O(n_GT_g + n_sT_s)$, which is similar to the time cost of the simplest DFKD method, DAFL \cite{chen2019data}. Furthermore, CuDFKD does not require any additional memory during its implementation.

\subsection{Does CuDFKD Converge?}

To ensure the convergence of the student model $\phi_s$, we should first find a suitable parameter $\theta^*$ to generate pseudo samples with varying degrees of difficulty, represented by the learned distribution $p_{\theta}(\mathbf{x})$. The reweighting factor $\mathbf{v}^*$ is then used to fine-tune the difficulty level. Finally, following the approach proposed in \cite{meng2017theoretical}, we analyze the convergence of CuDFKD using the following proposition:  

\begin{proposition}[Convergence of CuDFKD]
The objective function $\mathbf{F}(\mathbf{\phi}_s, \mathbf{v}, \theta)$ is equal to the Majorization Minimization(MM) step for optimizing the following latent function:
\begin{equation}
    \mathbf{F}_{\lambda}(\mathbf{L}(\theta, \mathbf{\phi}_s)) = \int_{0}^{\mathbf{L}} \mathbf{v}^*(\lambda, \mathbf{L}) d\mathbf{L}.
\end{equation}
Thus using AOS promises the convergence of all parameters.

\end{proposition}

\begin{proof}
According to Theorem 1 in \cite{meng2017theoretical}, and set parameter $\mathbf{w} = (\theta, \mathbf{\phi}_s)$, it is proved that

\begin{equation}\label{equ: obj}
    \begin{aligned}
    &\mathbf{F}_{\lambda}(\mathbf{L}(\mathbf{w})) \leq \mathbf{Q}(\mathbf{w} | \mathbf{w}^*) \\ &= F_{\lambda}(\mathbf{L}(\mathbf{w}^*)) + \mathbf{v}^*(\lambda, \mathbf{L}(\mathbf{w})) [\mathbf{L}(\mathbf{w}) - \mathbf{L}(\mathbf{w}^*)].
    \end{aligned}
\end{equation}

For any $\mathbf{w}^*$ and concave function $\mathbf{F}$. Thus for the \textit{majorization step}, the $\theta$ is fixed, and the $\mathbf{v}^*$ is achieved by $\mathbf{v}^*(\lambda, \mathcal{L}(\mathbf{w})) = \arg\min_{\mathbf{v} \in [0,1]^d} \mathbf{F}(\mathbf{w}, \mathbf{v})$, where $\mathbf{F}(\mathbf{w}, \mathbf{v})$ is defined in Equ. \ref{equ: spl}.

For the \textit{minimization step}, by setting the update step of $\mathbf{w}$ as $\mathbf{w}^{k+1} = \arg \min_{\mathbf{w}} \mathbf{Q}(\mathbf{w} | \mathbf{w}^{k})$, it's calculated by,

\begin{equation}\label{equ: proof_dfkd}
    \mathbf{w}^{k + 1} = \arg \min_{\mathbf{w}}  \mathbf{v}^*(\mathbf{L}(\mathbf{w}^k), \lambda)^T\mathbf{L}(\mathbf{w}).
\end{equation}

As the item loss function for student $\mathbf{L}(\mathbf{w})$ is defined as $\mathbb{E}_{x \sim p_{\theta}(\mathbf{z})}[D_{KL}(p_{\phi_t}(x) || p_{\phi_s}(x))]$, and $\mathbf{v}^*$ is fixed, it's equal to the objective function of KD. By the previous DFKD work, when setting data or diversity prior to the DFKD model, the AOS strategy promises convergence to optimize $\mathbf{w} = (\theta, \mathbf{\phi}_s)$.

According to MM theory\cite{caflisch1998monte}, the lower bound of the objective function on Equ. \ref{equ: obj} is monotonically non-decreasing and thus promises convergence.

\end{proof}

\subsection{Does CuDFKD work well?}
In this section, we provide a theoretical perspective on whether the easy-to-hard strategy improves the performance of DFKD. While we observed that CuDFKD might not perform as well when distilling small neural networks, we briefly discuss the results. By the VC theory of distillation \cite{wang2021knowledge}, the error between the teacher and student in CuDFKD is bounded by: 

\begin{equation}\label{equ: vckd}
\begin{aligned}
    R(f_s) - R(f_t) &\leq O(\frac{|\mathcal{F_s}|_C}{n^{\alpha}}) + \epsilon_l =\mathcal{L}_{gb, sr},
\end{aligned}
\end{equation}

Where $R()$ denotes the error of a specific function, and $f_s \in \mathcal{F}_s, f_t \in \mathcal{F}_t$ are the function set of student and teacher model, respectively. Both of them are related to approaching an unknown target function $f \in \mathcal{F}$. The parameter $\frac{1}{2} \leq \alpha \leq 1$ is related to the learning rate of the training model, and $\epsilon_l$ is the approximation error between the teacher and student functions. $n$ denotes the number of data samples. In the setting of DFKD, the learning rate $\alpha$ is manually set by providing a learning rate $lr_s$ for the KD training.

To avoid the student model being trapped into local minima during the early training, CuDFKD generates easy pseudo samples and utilizes the reweighting factor to adjust the difficulty level. It leads to a smaller value of $\mathcal{L}_{gb, sr}$ at an early stage of training, as the easy-to-hard strategy and the choice of divergence $D(z_t, z_s)$ result in $\alpha$ being close to $\frac{1}{2}$ and $\epsilon_l$ being small. As the student model grows, $|\mathcal{F}_s|_{C}$ also increases, increasing $\epsilon_l$ and $\alpha$, making it possible for the student model to generalize better.

The above discussion presents that while updating $\phi_s$, CuDFKD helps the student better approach the teacher. Similarly, the error gap between the teacher and target(ground truth) function is,
\begin{equation} \label{equ: tr}
    \begin{aligned}
    R(f_{t})-R(f_{r}) &\leq O\left(\frac{\left|\mathcal{F}_{t}\right|_{C}}{n^{\alpha_{t r}}}\right)+\epsilon_{t r} = \mathcal{L}_{gb, tr}.
    \end{aligned}
\end{equation}

The error gap between the target and student is calculated by adding Equ. \ref{equ: tr} and Equ. \ref{equ: vckd}, i.e., 

\begin{equation} \label{equ: sr}
\begin{aligned}
 R(f_s) - R(f_r) &\leq \mathcal{L}_{gb, tr} +  \mathcal{L}_{gb, st}.
\end{aligned}
\end{equation}

The component $\mathcal{L}_{gb, tr}$ solely depends on the generator's training performance. During the generation stage of CuDFKD, the optimal sample for the teacher is identified using the loss $\mathcal{L}_{oh}$, resulting in a low value of $\mathcal{L}_{gb, tr}$ during training. Consequently, this guarantees a lower bound on the sum of $\mathcal{L}_{gb, tr}$ and $\mathcal{L}_{gb, st}$ during training, ultimately leading to improved classification performance of the student model.

\subsection{Relationship with the Memory-based DFKD Methods}\label{subsec: memory}

CuDFKD shares similar concept with memory-based DFKD methods, as they both focus on improving the performance of the early student model. In this section, we explore the differences between CuDFKD and memory-based methods such as contrastive distillation (CMI) \cite{fang2021contrastive} and robust DFKD (PRE-DFKD) \cite{binici2022robust}. Typically, memory-based methods use past information during the generation stage to increase diversity and capture the learned information from the early steps. CMI and PRE-DFKD employ additional memory to store the generated pseudo samples from previous epochs, which helps the small model understand better. In contrast, CuDFKD uses dynamic difficulty and self-paced learning to generate pseudo samples that are more adaptable to the current state of the student model.

Some other continual learning methods, like \cite{wang2022learning,mazur2022target}, utilize regularization items to refine the preceding training information. Such methods assume that deep neural networks are overparameterized, and regularization helps to maintain the learned knowledge while preventing overfitting. While these methods differ from CuDFKD in their approach, they share a similar goal of improving the generalization ability of the student model. The formula is

\begin{equation}\label{equ: contL}
    L(\theta) = L_t(\theta) + \frac{\lambda}{2} \sum \Omega(\theta - \theta^*_{t-1})^2,
\end{equation}

$\tau$ represents the current training task, while $\tau-1$ denotes the task from the previous timestamp. The second item in Equ. \ref{equ: contL}, $\Omega(\theta - \theta^{t-1})^2$, relates to the preceding training epochs. Equ. \ref{equ: contL} bears similarity to Equ. \ref{equ: s}, as continual learning is a special case of self-paced learning by setting $v_{ContL}(\lambda, \mathbf{L}) = \mathbf{1}^{d}$ and $g_{ContL}(\lambda, \mathbf{v}) = \Omega(\theta - \theta^{t-1})^2$. It is worth noting that $g_{ContL}(\lambda, \mathbf{v})$ is convex when $\Omega$ is semi-positive definite. In continual learning, the regularizer term $g$ must incorporate parameters from previous epochs, while in self-paced learning, $\mathbf{v}$ and $g$ are defined by the current timestamp, $\tau$.

With the predefined $g$ in continual learning, CuDFKD can adjust the difficulty of samples early on in the training process, which can mitigate forgetting during short times. To enhance performance further, we aim to incorporate more information about early samples, which can better address the forgetting problem.

\section{Experiments} \label{sec: exp}
In this section, we present an evaluation of the performance and convergence of CuDFKD on various benchmarks and teacher-student pairs. We also conduct an ablation study to verify the effectiveness of the curriculum learning strategy. Additionally, we also explore the advantages of CuDFKD over other DFKD methods from several perspectives.

\subsection{Experimental Setup} \label{subsec:exp}

\begin{sidewaystable}[!htbp]
    \centering
    \caption{The DFKD performance on CIFAR10 and different teacher-student pairs. All the results are achieved from our implementation. Here \textit{WRN} is in short of \textit{Wider ResNet}. All chosen performances are the best performance during all runs of training. The metrics are defined in section \ref{subsec:exp}. Our CuDFKD model is bolded at the last line.}
    \label{table:main_result}
    \begin{threeparttable}[b]
    
    \begin{tabular}{c|ccc|ccc|ccc|ccc}
    \hline
    DataSet      & \multicolumn{12}{c}{CIFAR10}      \\ \hline
    Teacher      & \multicolumn{3}{c|}{ResNet34\tnote{1}}                                                                       & \multicolumn{3}{c|}{VGG11}                                                                          & \multicolumn{3}{c|}{WRN-40-2}                                                                       & \multicolumn{3}{c}{WRN-40-2}                                                                       \\
    Student      & \multicolumn{3}{c|}{ResNet18}                                                                       & \multicolumn{3}{c|}{ResNet18}                                                                       & \multicolumn{3}{c|}{WRN-40-1}                                                                       & \multicolumn{3}{c}{WRN-16-2}                                                                       \\ \hline
    Metric       &  Acc@1 & Agree@1 & $L_p$          &  Acc@1 & Agree@1 & $L_p$    &  Acc@1 & Agree@1 & $L_p$  &  Acc@1 & Agree@1 & $L_p$                  \\ \hline
    Vanilla Teacher   & 95.70                              & 100.00                          & 1.0000             & 92.25                              & 100.00                                   & 1.0000             & 94.87                              & 100.00                                   & 1.0000             & 94.87                              & 100.00                                   & 1.0000             \\
    Vanilla Student    & 94.23                              & 94.91                                    & 0.8276             & 94.23                              &          92.67              &          0.7868      & 91.21                              &       91.42          &       0.8046             & 90.42                              &           90.91                               &    0.7998               \\ \hline
    DAFL\cite{chen2019data}         & 91.45                              &  93.20             & 0.7686             & 81.10                              & 82.30                                    & 0.6303             & 81.33                              & 80.21                                    & 0.6445             & 81.55                              & 80.21                                    & 0.6445             \\
    ZSKT\cite{micaelli2019zero}         & 91.60                              & 93.57                                    & 0.7968             & 89.46                              &            90.24                             &           0.7193         & 86.07              &         88.42                                 &       0.6998             & 89.66                              & 91.19                                    & 0.7464             \\
    ADI\cite{yin2020dreaming}          & 93.26                              & 95.33                                    & 0.8478             & 90.36      &        93.11          & 0.8063       & 87.18                              & 88.53             &          0.7746         & 84.50                             &        85.61                                  &       0.7279             \\
    DFQ\cite{choi2020data}          & 94.61                              & 97.07                                    & 0.8492             & 90.84                              & 93.14                                    & 0.7568             & 91.69                              & 92.16                                    & 0.7380             & 92.01                              & 93.50                                    & 0.7826             \\
    CMI\cite{fang2021contrastive}          & 94.85                              & 96.46                                    & 0.8747             & 91.13                              & 94.18                                    & 0.8213             & 92.20                              & 93.59                                    & 0.8305             & 92.08                              & 94.01                                    & 0.8460             \\ 
    PRE-DFKD\cite{binici2022preventing}    &  91.65                              & 93.53                       & 0.8209 & 87.26                 & 89.88                      & 0.7527 & 86.68                 & 88.13                       & 0.7529 & 83.57                 & 84.67                      & 0.7146 \\ \hline
    \textbf{CuDFKD(Ours)} & \textbf{95.28}                     & \textbf{98.20}                           & \textbf{0.8915}    & \textbf{91.61}                     & \textbf{96.00}                           & \textbf{0.8267}    & \textbf{93.18}                     & \textbf{95.27}                           & \textbf{0.8440}    & \textbf{92.94}                     & \textbf{95.15}                           & \textbf{0.8477}    \\ \hline
    \end{tabular}
    
    \begin{tablenotes}
    \item[1] The result is implemented by ourselves in the same framework with CMI, by the same hyperparameter as the original paper for the previous DFKD methods. Actually, we cannot achieve the reported result from the original paper in CMI, and we report the highest performance over all trials of each method to promise the fairness of this experiment.
    \end{tablenotes}
    \end{threeparttable}
    
    \end{sidewaystable}

For the benchmarks, we use CIFAR10, CIFAR100\cite{krizhevsky2009learning}, and Tiny ImageNet\cite{le2015tiny}. The image size for CIFARs is $32 \times 32$, and for Tiny ImageNet is $64 \times 64$. For the models for distillation, we use ResNet\cite{he2016deep}, VGG\cite{simonyan2014very}, and Wider ResNet\cite{zagoruyko2016wide}. Please refer to table \ref{table:main_result} and \ref{table:main_result2} for different teacher-student pairs and other baseline methods. We use a generator to parameterize the generation process, i.e., $x = G(z)$. All experiments are implemented on NVIDIA 3090 TI GPUs.  For the baselines, we compare SOTA DFKD methods as DAFL\cite{chen2019data}, ZSKT\cite{micaelli2019zero}, ADI\cite{yin2020dreaming}, DFQ\cite{choi2020data}, CMI\cite{fang2021contrastive} and PRE-DFKD\cite{binici2022robust}. Our implementation is based on the framework provided by the paper from CMI\cite{fang2021contrastive}. Please refer to Appendix for detailed hyperparameters, the architecture of generators, and adversarial schedulers. We run all the methods for 250 epochs in CIFAR10, and 300 epochs in CIFAR100 and Tiny ImageNet and record the peak performance during all trainings. Our implementation is presented in \href{https://github.com/ljrprocc/DataFree}{https://github.com/ljrprocc/DataFree}. 

\begin{sidewaystable}[!htbp]
    \centering
    \caption{The DFKD performance on CIFAR100 and different teacher-student pairs. All results are achieved from our implementation. Here \textit{WRN} is in short of \textit{Wider ResNet}. All chosen performances are the best performance during all runs of training. The metrics are defined in section \ref{subsec:exp}. Our CuDFKD model is bolded at the last line.}
    \label{table:main_result2}
    \begin{threeparttable}[b]

    \begin{tabular}{c|ccc|ccc|ccc|ccc}
    \hline
    DataSet      & \multicolumn{12}{c}{CIFAR100}      \\ \hline
    Teacher      & \multicolumn{3}{c|}{ResNet34\tnote{1}}                                                                       & \multicolumn{3}{c|}{VGG11}                                                                          & \multicolumn{3}{c|}{WRN-40-2}                                                                       & \multicolumn{3}{c}{WRN-40-2}                                                                       \\
    Student      & \multicolumn{3}{c|}{ResNet18}                                                                       & \multicolumn{3}{c|}{ResNet18}                                                                       & \multicolumn{3}{c|}{WRN-40-1}                                                                       & \multicolumn{3}{c}{WRN-16-2}                                                                       \\ \hline
    Metric       &  Acc@1 & Agree@1 & $L_p$          &  Acc@1 & Agree@1 & $L_p$    &  Acc@1 & Agree@1 & $L_p$  &  Acc@1 & Agree@1 & $L_p$                  \\ \hline
    Vanilla Teacher    & 78.05                              & 100.00                                   & 1.0000             & 71.32                              & 100.00                                   & 1.0000             & 75.83                              & 100.00                                   & 1.0000             & 75.83                              & 100.00                                   & 1.0000             \\
    Vanilla Student    & 73.24                              & 76.69                                    & 0.5836             & 73.24                              & 72.12                                    & 0.5665             & 64.87                              & 65.90                                    & 0.5160             & 66.61                              & 67.85                                    & 0.5394             \\ \hline
    DAFL\cite{chen2019data}         & 67.58                              &  74.82             & 0.5232             & 64.49                              & 72.82                                    & 0.4891             & 55.06                              & 58.83                                    & 0.3847             & 55.48                              & 58.85                                    & 0.3878             \\
    ZSKT\cite{micaelli2019zero}         & 56.49                              & 61.61                                    & 0.4379             & 59.83                              & 67.75                                    & 0.4935             & 44.35                              & 46.56                                    & 0.3046             &           50.98            &  54.74            &  0.3695          \\
    ADI\cite{yin2020dreaming}          & 69.13                              & 75.56                                    & 0.5604             & 68.21                              & 76.79                                    & 0.5949             & 52.25                              & 54.74                                    & 0.4029             &   51.54       &  54.48                                        &     0.3966              \\
    DFQ\cite{choi2020data}          & 73.56                              & 82.62                                    & 0.6173             & 69.37                              & 80.16                                    & 0.5699             & 63.62                              & 68.42                                    & 0.4366             & 64.02                              & 68.51                                    & 0.4396             \\
    CMI\cite{fang2021contrastive}          & 75.16                              & 83.46                                    & 0.6551             & 69.78                              & 78.54                                    & 0.6158             & {66.89}                     & 71.69                                    & \textbf{0.5525}    & 66.07                              & 70.01                                    & 0.5289    \\
    PRE-DFKD\cite{binici2022preventing}     & 75.63                  & 87.07                     & 0.6574 & 70.29                 & 85.15                       & 0.6632 & 55.70                 & 59.01                      & 0.4075 & 49.54                & 52.29                      & 0.3538 \\ \hline
    \textbf{CuDFKD(Ours)} & \textbf{75.98 }                             & \textbf{87.89}                           & \textbf{0.6815}    & \textbf{71.22}                     & \textbf{85.85}                           & \textbf{0.6777}    & \textbf{67.16}                              & \textbf{72.53}                           & 0.5381            & \textbf{66.26}                     & \textbf{71.44}                           & \textbf{0.5295}             \\ \hline
    \end{tabular}
    
    \begin{tablenotes}
    \item[1] The result is implemented by ourselves in the same framework with CMI, by the same hyperparameter as the original paper for the previous DFKD methods. Actually, we cannot achieve the reported result from the original paper in CMI, and we report the highest performance over all trials of each method to promise the fairness of this experiment. Specifically, CMI\cite{fang2021contrastive} and DFQ\cite{choi2020data} report about 77\% in their paper, but we only get about 73.56\%, 75.16\% in 300 epochs.
    
    \end{tablenotes}
    \end{threeparttable}
    
    \end{sidewaystable}

As for the robustness and fidelity of the DFKD framework, we use metrics reported by \cite{stanton2021does,mirzadeh2020improved}, i.e., 

\begin{itemize}
    \item \textbf{Average top-1 accuracy $Acc@1$}. It's widely used for the measurement of the learning performance of student accuracy. It's calculated by
    \begin{equation}
        Acc@1 = \frac 1n \sum_{i=1}^n \mathbb{1}(\arg \max_j z_{i, j}^s = y_i),
    \end{equation}
    \item \textbf{Average top-1 agreement $Agree@1$}\cite{stanton2021does}. It's proposed to measure the generalization of the KD framework, and it's calculated by
    \begin{equation}
        Agree@1 = \frac 1n \sum_{i=1}^n \mathbb{1}(y_i^s = y_i^t),
    \end{equation}
    where $y_i^s = \arg \max_j z_{i, j}^s$ and $y_i^t = \arg \max_j z_{i, j} ^ t$ are the pseudo label of teacher and student output logits, respectively.

    \item \textbf{Average Probability Loyalty $L_p$}\cite{mirzadeh2020improved}.  It's first proposed to measure the performance of model compression of large pretrained language models. The probability loyalty measures the robustness of the learned distribution by the student model. It's calculated by
    \begin{equation}
        L_{p}(P \| Q)=1-\sqrt{D_{\mathrm{JS}}(P \| Q)},
    \end{equation}
    where $D_{\mathrm{JS}}(P \| Q)$ is the JS divergence between the distributions of $P$ and $Q$.

\end{itemize}

Here, the commonly used metric $Acc@1$ in previous DFKD methods \cite{binici2022robust,fang2021contrastive,fang2019data,luo2020large} only evaluates the gap between the distribution of student logits $p(z^s)$ and the label distribution $p(y)$. However, as the pretrained teacher model is the sole training target in the DFKD setup (discussed in Section \ref{sec: intro}), the generalization error of DFKD is determined by the gaps between teacher and student outputs. Therefore, alternative metrics such as $Agree@1$ and $L_p$ can more accurately assess these gaps and provide better performance measurements.

\subsection{Results and Analysis}

In this section, we report the result and convergence of CuDFKD. In addition, we compare the performance and convergence of CuDFKD with previous DFKD methods.

\subsubsection{Results on Different Benchmarks}

Table \ref{table:main_result} and Table \ref{table:main_result2} present the performance of different benchmarks and teacher-student pairs. The labels \textit{Vanilla Teacher} and \textit{Vanilla Student} refer to the training accuracy of the teacher and student models, respectively, obtained by training them from scratch using the original training data. The metrics $Agree@1$ and $L_p$ for the teacher model are 100.00\% and 1.0000 respectively, according to their definitions. From Table \ref{table:main_result}, we observe that CuDFKD outperforms all the methods. Furthermore, memory-based methods such as CMI \cite{fang2021contrastive} and PRE-DFKD \cite{binici2022robust} require an extra memory bank or memory generator, while our CuDFKD only adds a scheduler for the training. For different teacher-student pairs in CIFAR10, we achieve the best overall performance metrics, as shown in Table \ref{table:main_result}. It is worth noting that we use the same hyperparameter settings and code implementation as the original paper by CMI\footnote{\href{https://github.com/zju-vipa/CMI}{https://github.com/zju-vipa/CMI}}.

\begin{figure*}[!htbp]
     \centering
     \begin{subfigure}[b]{0.23\textwidth}
         \centering
         \includegraphics[width=\textwidth]{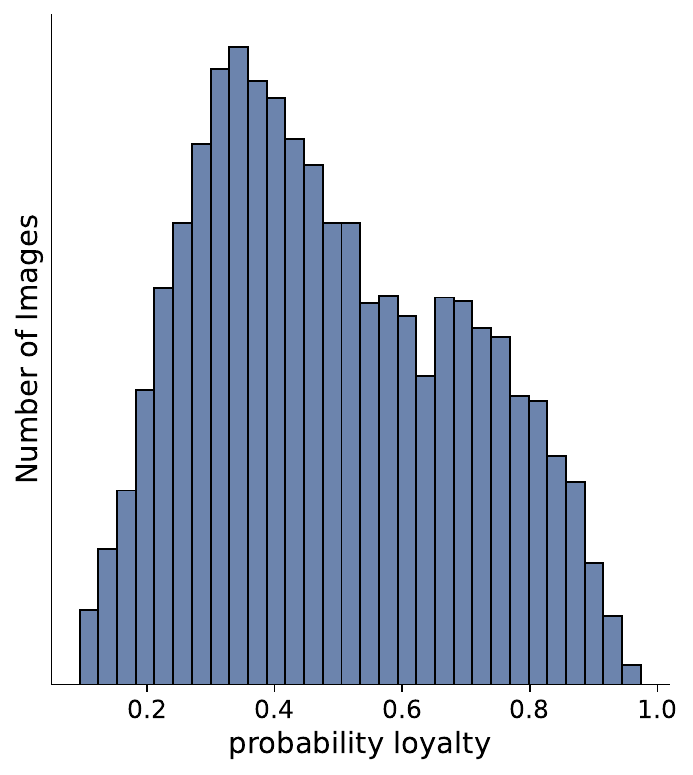}
         \caption{$L_p$ distribution of DAFL.}
         \label{fig:dafl loyalty}
     \end{subfigure}
     \hfill
     \begin{subfigure}[b]{0.23\textwidth}
         \centering
         \includegraphics[width=\textwidth]{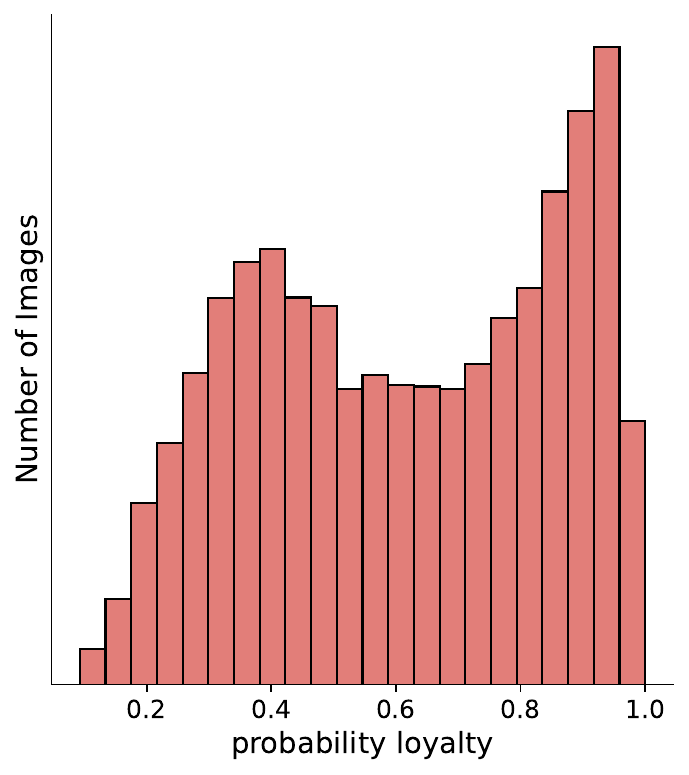}
         \caption{$L_p$ distribution of CMI.}
         \label{fig:cmi loyalty}
     \end{subfigure}
     \hfill
     \begin{subfigure}[b]{0.23\textwidth}
         \centering
         \includegraphics[width=\textwidth]{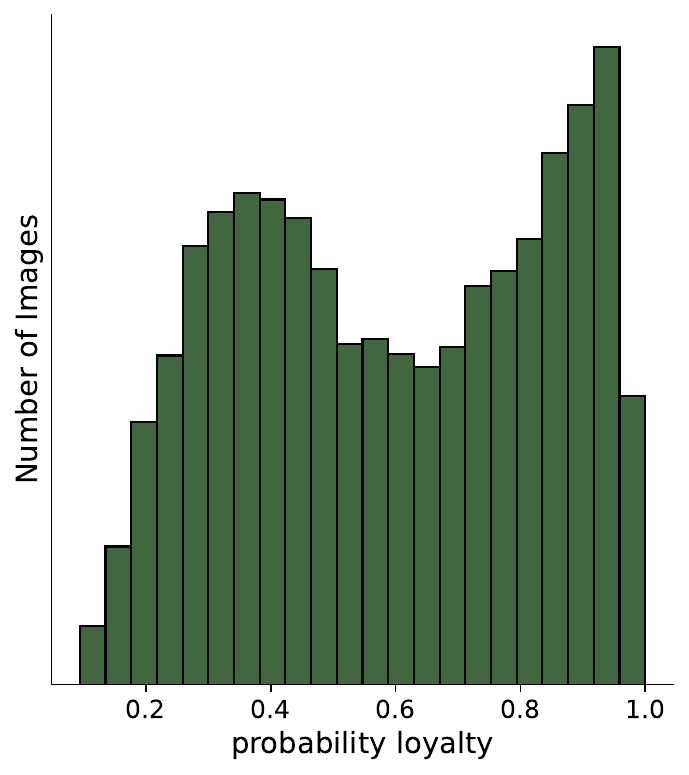}
         \caption{$L_p$ distribution of ADI.}
         \label{fig:deepinv loyalty}
     \end{subfigure}
     \hfill
    \begin{subfigure}[b]{0.23\textwidth}
         \centering
         \includegraphics[width=\textwidth]{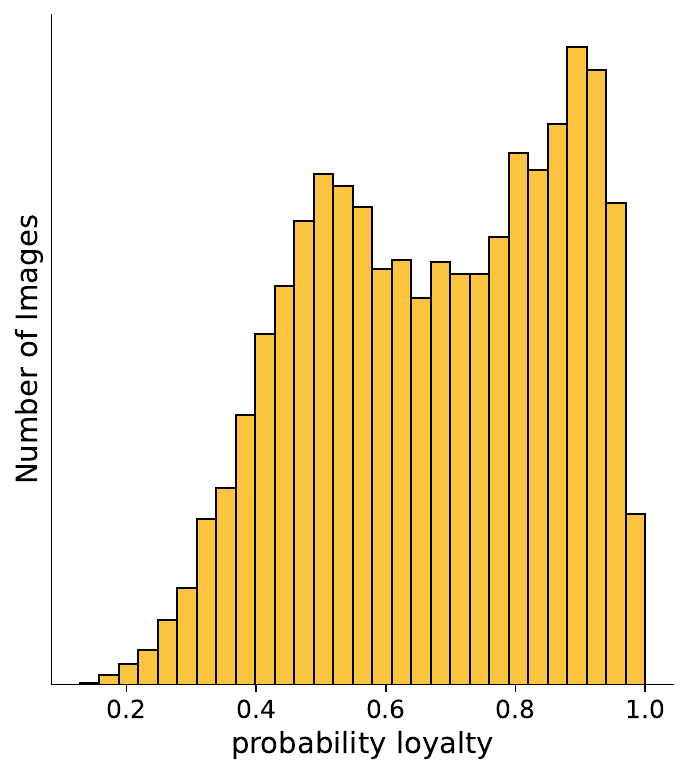}
         \caption{$L_p$ distribution of CuDFKD.}
         \label{fig:cudfkd loyalty}
     \end{subfigure}
        \caption{Distribution of probability loyalty $L_p$ for different DFKD methods. T: VGG11, S: ResNet18, Benchmark: CIFAR100. Better viewed in color.}
        \label{fig:loyalty distribution}
\end{figure*}

\begin{table}[!htbp]
\centering
\caption{The result of Tiny ImageNet. All results are achieved by our reimplementation\tnote{1}. Methods marked by * mean that the validation performance does not increase for a long time.}
\label{table: timgnet}
\begin{threeparttable}
\begin{tabular}{cccc}
\hline
Method  & Acc@1 & Agree@1 & $L_p$ \\ \hline
Van. T & 61.47  & 100.00 & 1.0000 \\
Van. S & 43.18  & 44.10 & 0.3502   \\ \hline
CMI*    &  11.65 $\pm$ 1.23   & 16.72 $\pm$ 1.33 & 0.2003 $\pm$ 0.0204 \\
ADI*     &   26.00 $\pm$ 0.87   & 36.98 $\pm$ 0.82 & 0.2736 $\pm$ 0.0185 \\ 
DFQ     &   41.30  & 47.07 & 0.3300 \\ \hline
CuDFKD  &   \textbf{43.42} & \textbf{50.07} & \textbf{0.3562} \\ \hline
\end{tabular}
\begin{tablenotes}
\footnotesize
\item[1] The result is achieved by the same hyperparameter as the CMI original paper in CMI code implementation. As the teacher is retrained in this paper, results are different from those reported in the original paper. For the fairness of two methods, we run 3 experiments and report the error for CMI and ADI.
\end{tablenotes}
\end{threeparttable}
\end{table}

For harder benchmarks, we also implement CuDFKD on Tiny ImageNet, and the result is presented in Table \ref{table: timgnet}. The larger image resolution of Tiny ImageNet from $32 \times 32$ to $64 \times 64$ required us to adjust the hyperparameters of the generator models from the original paper and freeze them. Additionally, we adjust the hyperparameters of the dynamic strategy further $\mathbf{v}^*$ and retrain the teacher model on the benchmark before performing CuDFKD. For comparison with other DFKD methods, we used the same generator architecture as CuDFKD. From Table \ref{table: timgnet}, we observe that CMI and ADI did not perform well in our implementation, while CuDFKD maintained stable and optimal performance, even in the face of more challenging benchmarks.

\subsubsection{Robustness at the evaluation sets} To evaluate the robustness and fidelity of CuDFKD, we compute two metrics, namely the average Top-1 agreement ($Agree@1$) and the average probability loyalty ($L_p$). Tables \ref{table:main_result} and \ref{table:main_result2} show the results, from which we make the following observations:

1) Both $Agree@1$ and $L_p$ metrics are positively correlated with $Acc@1$.

2) Among all the benchmarks and student-teacher pairs, CuDFKD achieves the best performance, which can be attributed to the full learning of the information in the teacher model, i.e., the full imitation of the teacher.

3) Although \textit{Vanilla Student} results in a more desirable $Acc@1$, its $Agree@1$ and $L_p$ values are not satisfactory.

Observation 1) implies that the performance of the student model in the DFKD task relies on the degree of imitation of the teacher model. Observation 2) explains that the main reason for the superior performance of CuDFKD is the adequate learning of the teacher model. Finally, observation 3) indicates that such approximation to teacher is gradually achieved during the training of the student model. The student model approaches the target distribution $p(y)$. By measuring $Agree@1$ and $L_p$, we emphasize the importance of approximating the distribution determined by the teacher model for DFKD training and thus present the superiority of our CuDFKD over previous methods.

\begin{figure*}[!htbp]
    \centering
    \begin{subfigure}[b]{0.47\textwidth}
        \centering
        \includegraphics[width=\textwidth]{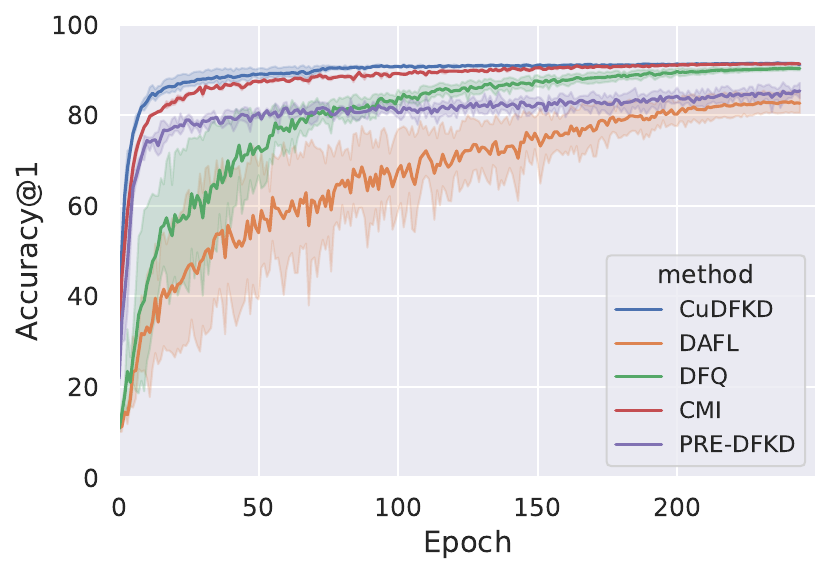}
        \caption{The convergence for CIFAR10, the running epochs is 250. T: VGG11, S: ResNet18.}
        \label{fig:cifar10 convergence}
    \end{subfigure}
    \hfill
    \begin{subfigure}[b]{0.47\textwidth}
        \centering
        \includegraphics[width=\textwidth]{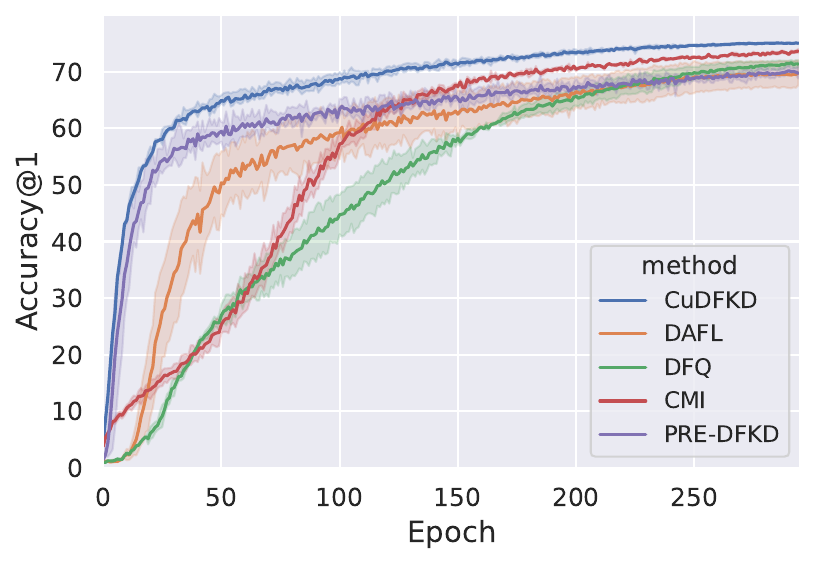}
        \caption{The convergence for CIFAR100, the running epochs is 300. T:ResNet34, S: ResNet18, }
        \label{fig:cifar100 convergence}
    \end{subfigure}
       \caption{Visualization of validation accuracy curves of different curves with different methods. The blue curve is our CuDFKD, respectively. Better viewed in color.}
       \label{fig:convergence}
\end{figure*}

To further investigate the generalization performance of DFKDs, we visualized the distribution of $L_p$ at the final epoch of ADI, CMI, and CuDFKD using {VGG11-ResNet18} pairs on CIFAR100. The resulting distribution, as shown in Figure \ref{fig:loyalty distribution}, demonstrates that CuDFKD achieves higher probabilistic loyalty compared to ADI, DAFL, and CMI on the CIFAR100 benchmark, with a clear rightward shift of the first wave peak. However, CMI generates more low $L_p$ samples than higher one, leading to a lower mean $L_p$. It is expected since ADI and DAFL introduces a regularization term to the generator's loss function $L_g$ to increase the diversity of the pseudo-samples, which trades off similarity to the teachers for the generation of diverse samples.

For CuDFKD, the difficulty function is designed based on the disagreement between teachers and students, and it changes dynamically with the training process. Therefore, it is more inclined towards the imitation and learning process for the teacher, resulting in higher $Agree@1$ and $L_p$ performance. The performance of $Agree@1$ and $L_p$ supports our hypothesis that $\mathcal{L}_{gb, st}$ is low, while high $Acc@1$ supports the hypothesis that $\mathcal{L}_{gb, st} + \mathcal{L}_{gb, tr}$ is low, as shown in Equ. \ref{equ: sr}. In conclusion, CuDFKD achieves comparable results among SOTA methods because it better explores the global minimum during the easy-to-hard training strategy. Note that these results were achieved at the best checkpoint during training, which we refer to as the \textbf{peak performance}. We also evaluated the convergence and performance of CuDFKD during the training stage.

\subsubsection{Convergence of CuDFKD}\label{sec: conv}

In this section, we check the convergence of CuDFKD and compare it with other SOTA DFKD methods by plotting the curve of validation accuracy at each timestamp during training. We use VGG11-ResNet18 and ResNet34-ResNet18 as teacher-student pairs. Hyperparameters are those specified in the original papers of each method. The metric used for this experiment is $Acc@1$, and the same curriculum strategy as in Table \ref{table:main_result} is used. Figure \ref{fig:convergence} show that CuDFKD converges quickly and efficiently, achieving a 90\% top-1 accuracy on CIFAR10 in less than 15 epochs and stably improving until the end of training. Moreover, the training process is stable, as the validation accuracy does not fluctuate significantly during the training process. In contrast, previous DFKD methods exhibit different stability, making it difficult to achieve optimal performance.

\begin{table*}[!htbp]
\centering
\caption{The stability for the training of ResNet18 taught by ResNet34 on two CIFAR benchmarks. Here the $\mu$ and $\sigma^2$ achieved by CuDFKD are performed by 4 runs. The memory usage and training time tests are all performed on one NVIDIA 3090 GPU, with batch size of 256. Here we use $Acc@1$ as the metric. }
\label{table: conv}
\begin{tabular}{c|ccccc|ccccc}
\hline
\multirow{2}{*}{DataSet} & \multicolumn{5}{c|}{CIFAR10} & \multicolumn{5}{c}{CIFAR100} \\ \cline{2-11} 
 & $\mu$  & $\sigma^2$  & best Acc@1 & GPU Memory  & Time & $\mu$  & $\sigma^2$  & best Acc@1 & GPU Memory & Time \\ \hline
DAFL        &   62.6   &  17.1     &  92.0     &\textbf{ 6.45G } & 6.10h &  52.5 &  12.8 &  74.5  &  \textbf{6.45G}  & \textbf{7.09h}      \\
DFAD\cite{fang2019data}      &    86.1    &   12.3     &  93.3     & -   &    - & 54.9 & 12.9 &  67.7    & -   &    -   \\
ADI\cite{yin2020dreaming}   &   87.2     & 13.9    &   93.3    &  7.85G   &   25.2h & 51.3  & 18.2  &  61.3   & 7.85G   &   30.4h    \\
CMI\cite{fang2021contrastive}      &   82.4     & 16.6  &   94.8    &    12.5G   & 13.3h & 55.2 & 24.1 &  77.0  &     12.5G        &   22.3h   \\
MB-DFKD\cite{binici2022preventing}      &  83.3  & 16.4  &  92.4   &   - &    -   &   64.4     &   18.3    &     75.4   &      -       &  -     \\
PRE-DFKD\cite{binici2022robust}    & 87.4 & 10.3  &  94.1  &  -    &  -      &    70.2    &    11.1   &  \textbf{ 77.1 }    &       -      &   -    \\ \hline
CuDFKD(Ours)      &   \textbf{94.1}     &    \textbf{2.88}    & \textbf{95.0} & 6.84G  & \textbf{5.48h}  &  \textbf{71.7}      & \textbf{4.37}  & 75.2   & 6.84G   &    7.50h   \\ \hline
\end{tabular}
\end{table*}

Quantitatively, we assess the stability of our framework and compare it with other SOTA DFKD methods by computing the mean $\mu[s_{acc}]$ and variance $\sigma^2[s_{acc}]$ of the validation student accuracy over all epochs, as motivated by PRE-DFKD \cite{binici2022robust}. Specifically, we calculate $\mu[s_{acc}]$ and $\sigma^2[s_{acc}]$ as the mean and variance of the validation accuracy from the 80th epoch to the end of the training, as all methods tend to converge from this epoch. Moreover, we record the peak performance during training, the used GPU memory, and training time. Table \ref{table: conv} presents the $\mu[s_{acc}]$ and $\sigma^2[s_{acc}]$ for all SOTA methods. Our CuDFKD shows well stability during DFKD training without requiring any memory bank or replay generators, achieving higher $\mu$ than all other methods. Furthermore, the memory usage and time cost indicate that CuDFKD uses the same memory as the simplest DFKD method as DAFL while achieving significantly better performance.

\begin{figure*}[!htbp]
    \centering
    \begin{subfigure}[b]{0.2 \textwidth}
        \centering
        \includegraphics[width=\textwidth]{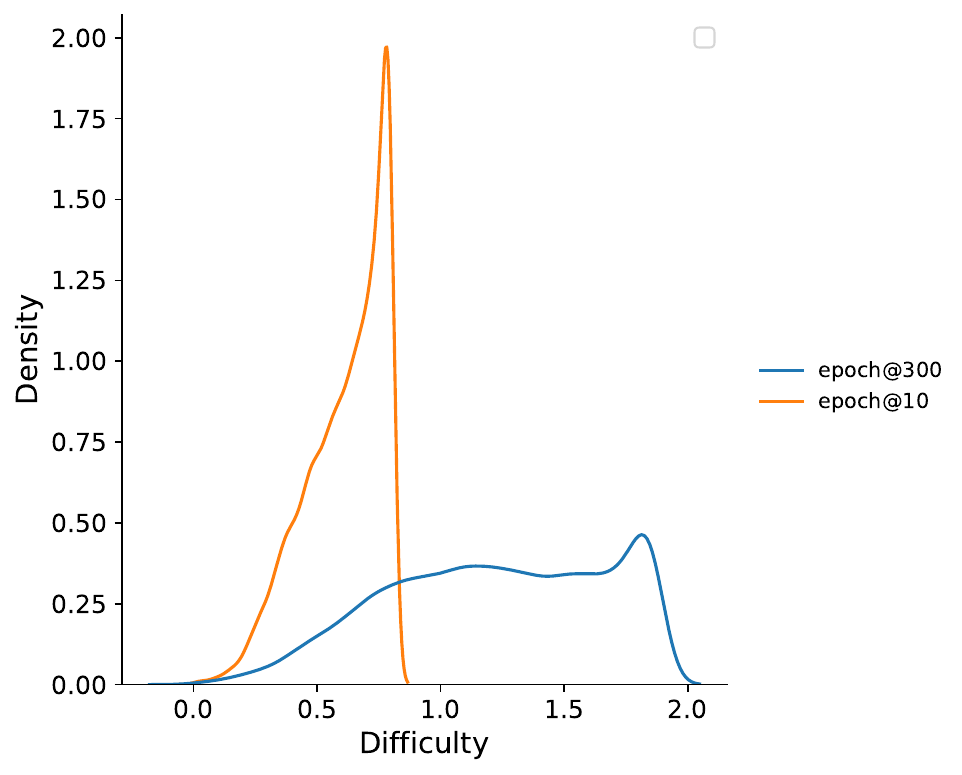}
        \caption{Distribution of difficulty at different timestamps.}
        \label{fig:diff}
    \end{subfigure}
    \hfill
    \begin{subfigure}[b]{0.37 \textwidth}
        \centering
        \includegraphics[width=\textwidth]{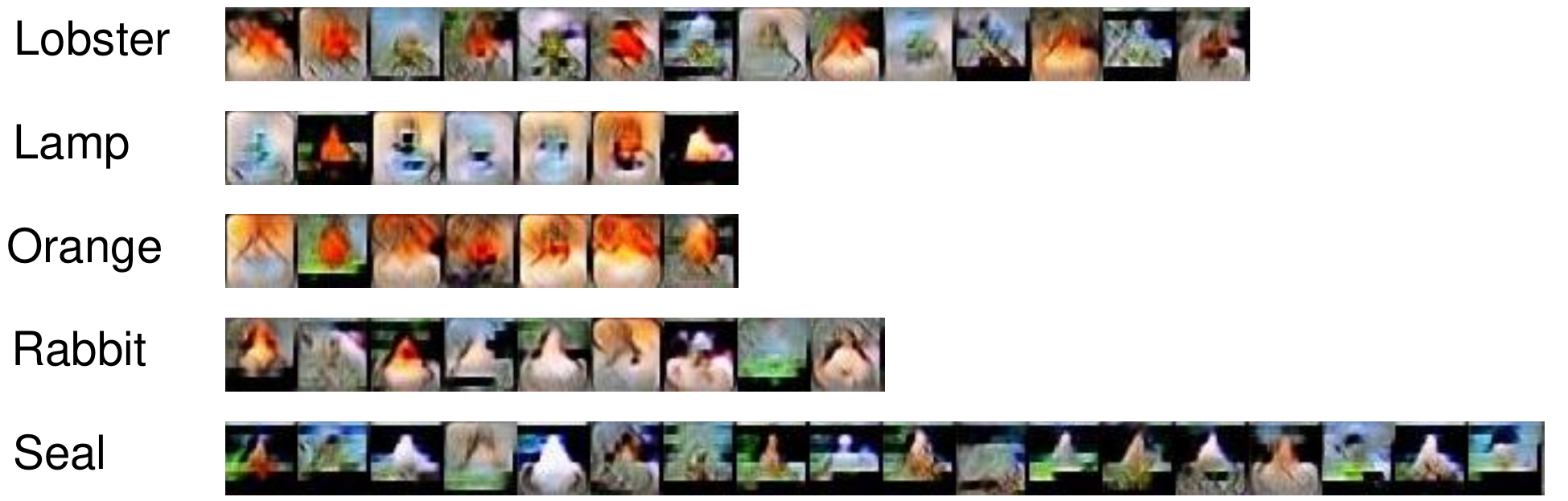}
        \caption{Visualization of generated samples by the class at epoch 10.}
        \label{fig:epoch at 10}
    \end{subfigure}
    \hfill
   \begin{subfigure}[b]{0.37\textwidth}
        \centering
        \includegraphics[width=\textwidth]{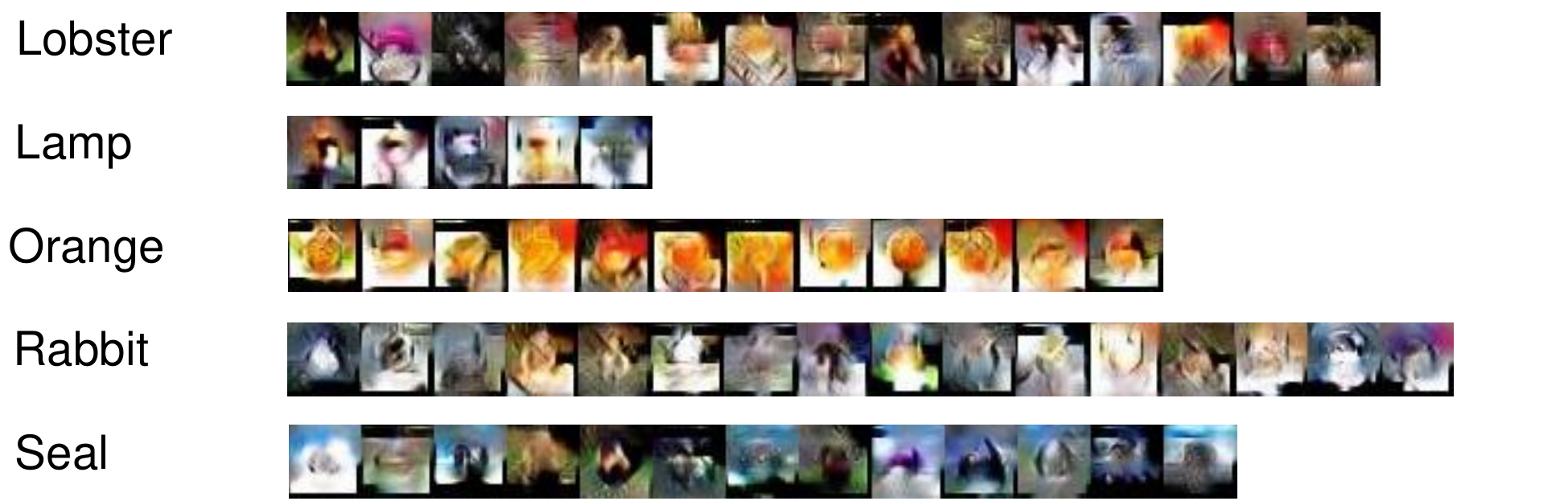}
        \caption{Visualization of generated samples by the class at epoch 300.}
        \label{fig:epoch at 300}
    \end{subfigure}
       \caption{Visualization of the distribution of difficulty at different stages while training CuDFKD. T: ResNet34, S: ResNet18, Benchmark: CIFAR100. Better viewed in color.}
       \label{fig:difficulty visualization}

\end{figure*}

\begin{figure*}[!htbp]
    \centering
    \begin{subfigure}[b]{0.32\textwidth}
        \centering
        \includegraphics[width=\textwidth]{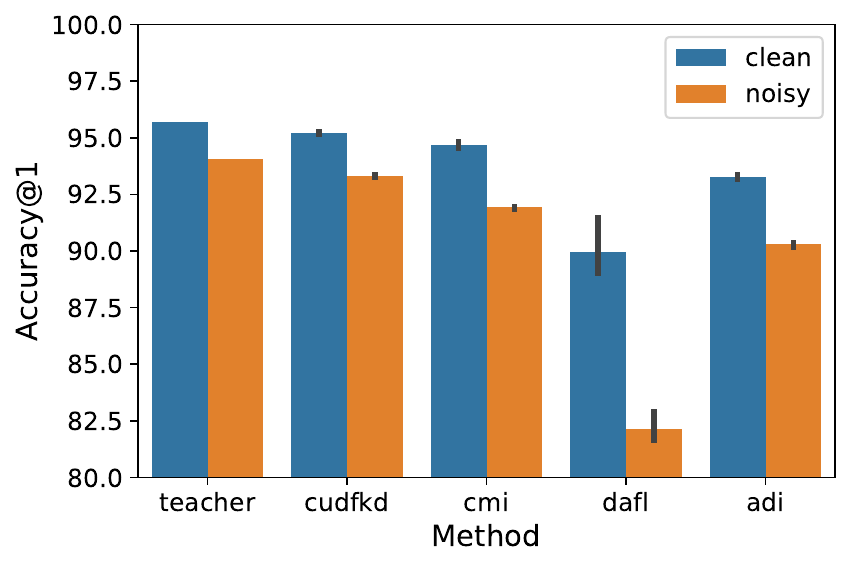}
        \caption{$Acc@1$.}
        \label{fig:acc1}
    \end{subfigure}
    \hfill
    \begin{subfigure}[b]{0.32\textwidth}
        \centering
        \includegraphics[width=\textwidth]{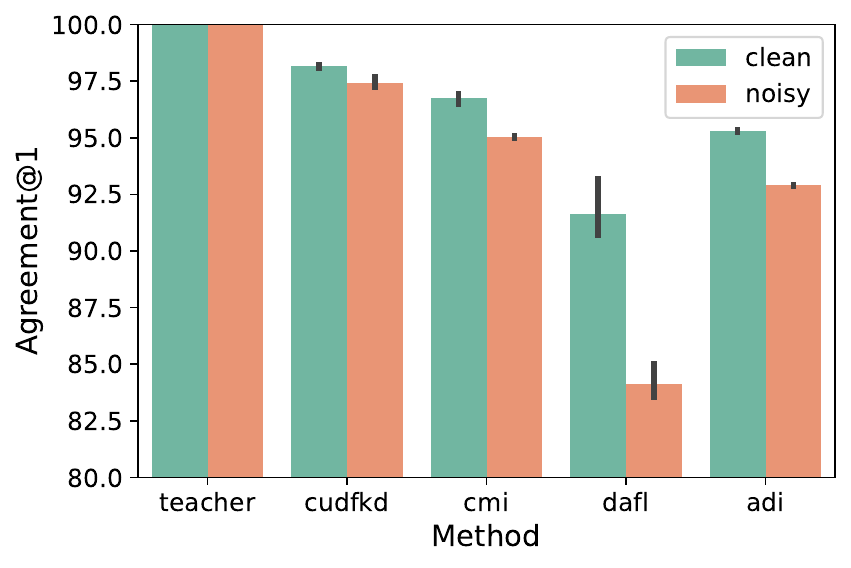}
        \caption{$Agree@1$.}
        \label{fig:agree1}
    \end{subfigure}
    \hfill
   \begin{subfigure}[b]{0.32\textwidth}
        \centering
        \includegraphics[width=\textwidth]{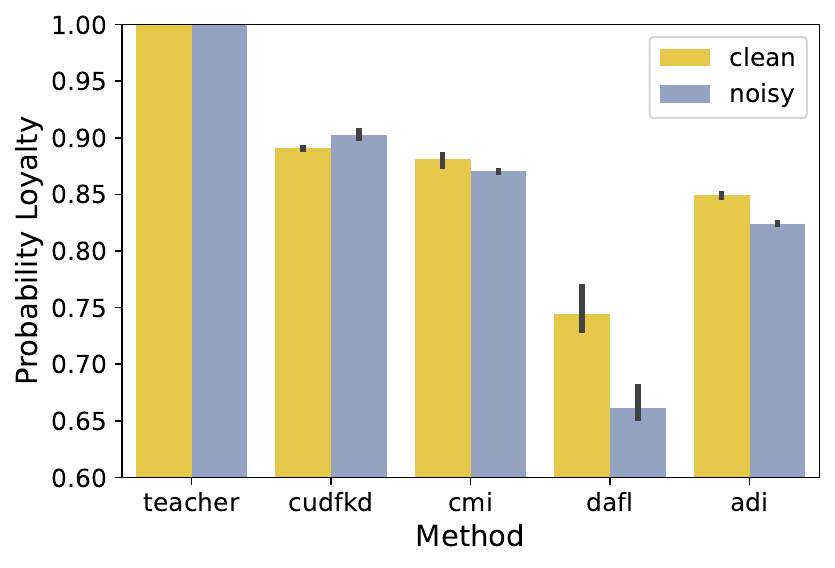}
        \caption{$L_p$.}
        \label{fig:loyalty}
    \end{subfigure}
       \caption{Performance of using clean and noisy teacher for training CuDFKD on different metrics. T: ResNet34, S: ResNet18, Benchmark: CIFAR10. Better viewed in color.}
       \label{fig: noisy}
\end{figure*}



\subsection{Difficulty Visualization at Different Stages}

In CuDFKD, we utilize the divergence between the teacher and student models to define the difficulty level for each training sample. It is important to ensure that the easy-to-hard criterion is satisfied during training. To investigate the impact of this design on student model training, we visualize the sample difficulty distribution at different training stages, along with the corresponding images. We focus on visualizing image $\mathbf{X}$ as the difficulty, consistent with the definition of SPL. We also compute the difficulty of the generated samples by finding the label $y_t$ with the highest score: $y_t = \arg \max_{i} z_{i}^t$. The experiments are conducted on CIFAR100, with WRN-40-2 as the teacher and WRN-16-2 as the student, using a batch size of 1024.

The visualized samples are presented in Figure \ref{fig:difficulty visualization}. Figure \ref{fig:diff} shows the visualization results of the difficulty distribution for the 10th epoch (distribution B) as well as for the 300th epoch (distribution A). Figures \ref{fig:epoch at 10} and \ref{fig:epoch at 300} show the visualization results of the same category of pseudo-samples for the 10th epoch (Figure \ref{fig:epoch at 10}) as well as for the 300th epoch (Figure \ref{fig:epoch at 300}). We observe sharper peaks in distribution A at the beginning of training compared to distribution B, and the peaks in A are more to the left than those in B. These sharp peaks indicate that the samples are not sufficiently diverse at the beginning of the training phase, and therefore the visualization is not as effective. Similar results are obtained from the sample visualization results in Figures \ref{fig:epoch at 10} and \ref{fig:epoch at 300}. For Figure \ref{fig:epoch at 10}, the samples between different classes already possess a certain gap, however, the intra-class samples are more similar. This setting for the samples can effectively help students learn quickly from the teacher model and achieve the learning process from easy to difficult, as explained in Section \ref{sec: intro} and Figure \ref{fig:workflow}.

\subsection{Comparable results on Noisy Teachers} 
In this section, we investigate the robustness of CuDFKD in learning from a noisy teacher model. They create the noisy teacher model by training a ResNet34 model for only 200 epochs with an initial learning rate of 1.0 using a cosine scheduler and without finetuning other parameters. The teacher model achieves 94.07\% top-1 accuracy but contains some incorrect classification information, which introduces noise into the distillation process. The experiments are conducted with a batch size of 768, and the results are presented in Figure \ref{fig: noisy}. In addition, we run three experiments with random seeds $\{0, 10, 20\}$ and report the error bars in the figure for all three metrics.

From Figure \ref{fig: noisy}, our CuDFKD model achieves more than 93.3\% $Acc@1$ with the guidance of the noisy teacher, experiencing less than a 0.7 percentage drop compared to the teacher model. Moreover, CuDFKD improves the $L_p$ of student guided by the noisy teacher, indicating its robustness against noise in the teacher. Such robustness is attributed to the dynamic reweighting of generated pseudo samples, which provides another form of supervision for the student. 

Additionally, we also have some interesting discoveries from Figure \ref{fig: noisy}. First, DAFL suffers significant drops in all metrics due to its failure to learn adversarial samples, resulting in its lack of robustness. Second, CuDFKD and CMI perform better than ADI and DAFL in terms of the drop in accuracy, as the memory in CMI improve the interaction between the teacher and student. Last, CuDFKD reports less variance than CMI across different random seeds and even better $L_P$ when using the noisy teacher model, indicating that a good dynamic strategy can improve the efficiency of knowledge transfer and align the output distribution from the teacher and student models. Therefore, training DFKD on the noisy teacher model presents an interesting area for future research.

\subsection{Ablation Study on Dynamic Modules}
This section investigates the impact of different dynamic modules in CuDFKD. As discussed in section \ref{sec: theo}, CuDFKD employs a dynamic learning strategy in the generation and training stages. Specifically, we adjust the hyperparameter $\alpha_{adv}$ to control the generation target and the reweighting factor $\mathbf{v}$ to maintain the target of KD. Table \ref{table: abl} presents the performance of CuDFKD with and without different dynamic modules. To assess the convergence rate, we measure the time it takes for the student model to achieve 90\% validation $Acc@1$, denoted as $t_{90}$.

\begin{table}[!htbp]
    \centering
    \caption{Ablation study of different components of CuDFKD. If adding the reweighting factor $\mathbf{v}$, we mark $\checkmark$ in the $\mathbf{v}$ column. The last column $t_{90}$ represents the cost of GPU hours for the student model to reach 90\% top-1 validation accuracy. T: ResNet34, S: ResNet18, Benchmark: CIFAR10. }
    \label{table: abl}
    \begin{tabular}{ll|llll}
    \hline
    $\mathbf{v}$ & $\alpha_{adv}(\tau)$                         & Acc@1 & Agree@1 & $L_p$  & $t_{90}$ \\ \hline & 1.0                                         & 94.61                        & 97.07                          & 0.8492 & 3.98h    \\
    $\checkmark$ & 1.0                                    & 94.86                        & 97.59                          & 0.8805 & 1.21h    \\
        & Equ \ref{equ: grad_adv} & 94.79 & 97.51 & 0.8806 & 3.28h \\

    $\checkmark$ & Equ \ref{equ: grad_adv} & \textbf{95.28}                        & \textbf{98.20}                          & \textbf{0.8915} & \textbf{0.71h}    \\ \hline
    \end{tabular}
    \end{table}

Table \ref{table: abl} presents the performance of CuDFKD with different dynamic modules. Compared with DFQ\cite{choi2020data} (line 1), adding the reweighting factor $\mathbf{v}$ (line 2) leads to significant improvements in $Agree@1$ and $L_p$. It presents that t$\mathbf{v}$ helps the student model approach the teacher model. Furthermore, the convergence rate of DFKD significantly improves as $t_{90}$ decreases from nearly 4 GPU hours to 1.21 GPU hours. Using the only dynamic generation target by adjusting $\alpha_{adv}$ (line 3) results in significant improvements in $Agree@1$ and $L_p$ but with a slight decrease in $t_{90}$. Finally, combining both dynamic modules (line 4) leads to substantial improvements in all metrics, with an even greater decrease in $t_{90}$. We conclude that the dynamic generation target by $\alpha_{adv}(\tau)$ can help the reweighting factor $\mathbf{v}$ find the global minima of the KD loss function $\mathbf{L}$ faster. Therefore, both designed dynamic modules are crucial for CuDFKD's performance.


\begin{figure}[!htbp]
    \centering
    \begin{subfigure}[b]{0.245\textwidth}
        \centering
        \includegraphics[width=\textwidth]{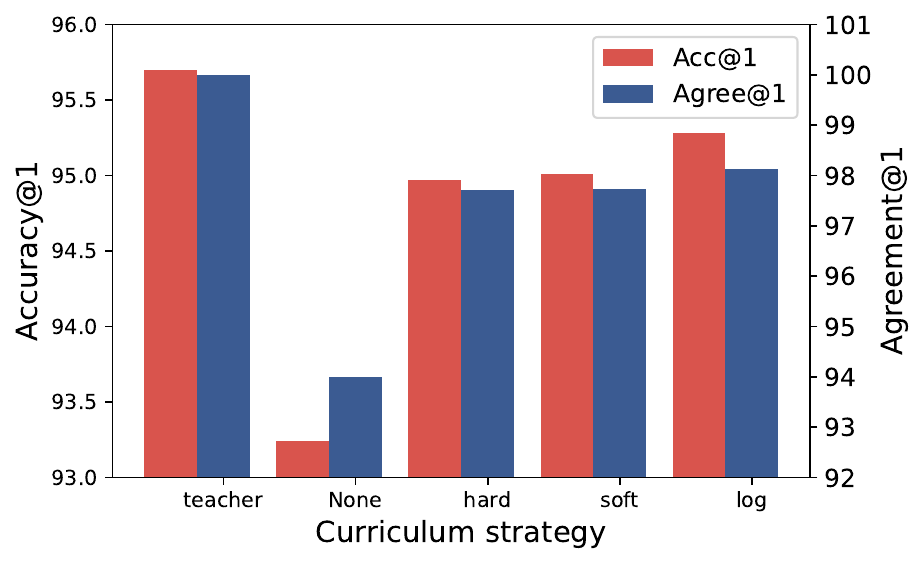}
        \caption{$Acc@1$ and $Agree@1$ with different settings of $\mathbf{v}^*$s.}
        \label{fig:strategy_metric1}
    \end{subfigure}
    \hfill
    \begin{subfigure}[b]{0.225\textwidth}
        \centering
        \includegraphics[width=\textwidth]{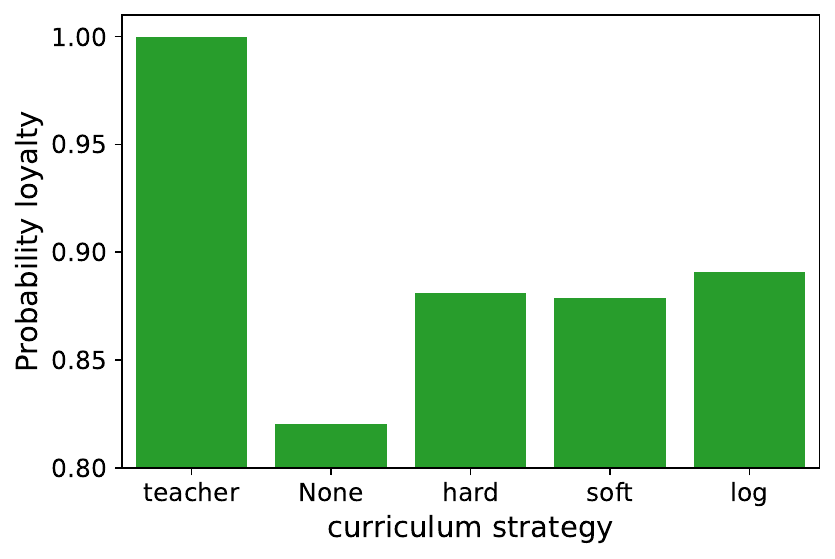}
        \caption{$L_p$ with different settings of $\mathbf{v}^*$s.}
        \label{fig:strategy_metric2}
    \end{subfigure}
       \caption{The performance of CuDFKD with different settings of dynamic strategy. Teacher: ResNet34, Student: ResNet18, Benchmark: CIFAR10. Better viewed in color.}
       \label{fig:strategy}
\end{figure}

\begin{figure}[!htbp]
    \centering
    \begin{subfigure}[b]{0.25\textwidth}
        \centering
        \includegraphics[width=\textwidth]{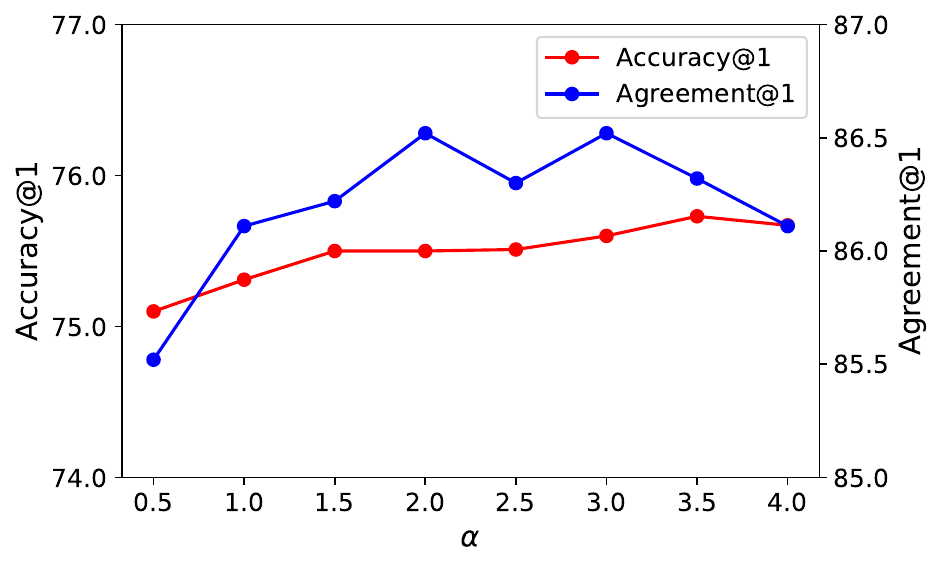}
        \caption{$Acc@1$ and $Agree@1$ with different $\alpha$s. }
        \label{fig:strategy_metrice}
    \end{subfigure}
    \hfill
    \begin{subfigure}[b]{0.225\textwidth}
        \centering
        \includegraphics[width=\textwidth]{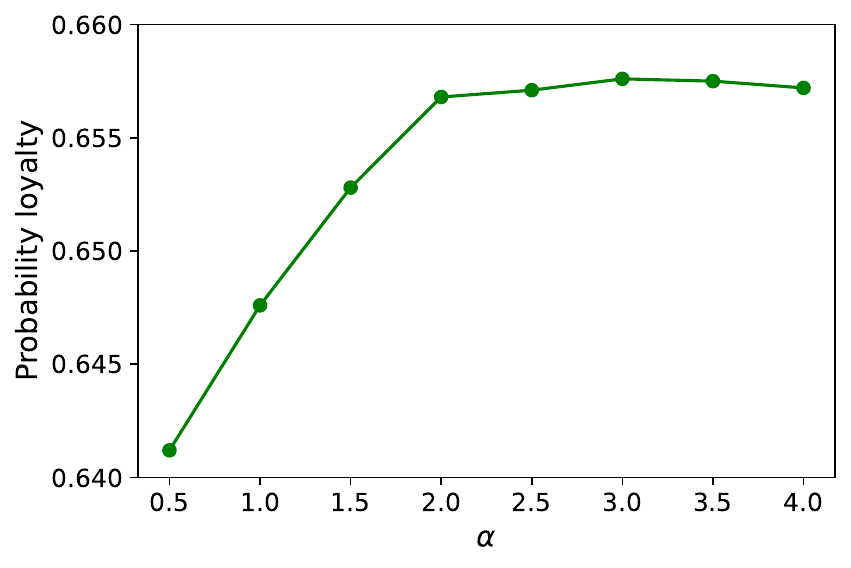}
        \caption{$L_p$ with different settings of $\alpha$s.}
        \label{fig:strategy_metricf}
    \end{subfigure}
       \caption{The performance of CuDFKD with different settings of $\alpha$. Teacher: ResNet34, Student: ResNet18, Benchmark: CIFAR100. Better viewed in color.}
       \label{fig:alpha}
\end{figure}

\subsection{Hyperparmeter Sensitivity}\label{subsec: hyper}

In this section, we check the effectiveness of CuDFKD and compare the performance of different hyperparameters for a given design of $\mathbf{v}^*$ and $\mathcal{L}_\theta(x)$. Furthermore, we explore various hyperparameters for the dynamic modules at different stages.

\textbf{Effect of dynamic training module, $\mathbf{v}^*$.} Specifically, we check different designs of $\mathbf{v}^*$. As it has the analytical solution, we design different closed forms of $\mathbf{v}^*$s by previous work\cite{wang2021survey}, i.e.,

\begin{itemize}
    \item \textbf{Hard} : $\mathbf{v}^*(
    \lambda, \mathbf{L}) = \mathbf{1}(\mathbf{L} < \lambda)$;
    \item \textbf{Soft} : $\mathbf{v}^*(
    \lambda, \mathbf{L}) = \mathbf{1}(\mathbf{L} < \lambda)(1-\mathbf{L} / \lambda)$;
    \item \textbf{Logarithm} : $\mathbf{v}^*(
    \lambda, \mathbf{L}) = \frac{1+e^{-\lambda}}{1+e^{\mathbf{L}-\lambda}}$,
\end{itemize}

The function $\mathbf{1}()$ represents the indicator function. The remaining hyperparameters are kept consistent with those used in previous experiments. All experiments are implemented on the distillation ResNet34-ResNet18 in CIFAR10, and they're summarized in Figure \ref{fig:strategy}. 

The findings meet with those of prior studies on curriculum learning \cite{wang2021survey}. As illustrated in Figure \ref{fig:strategy}, we observe the following: 1) Regardless of the dynamic strategy employed, the $Acc@1$ of the student models shows improvements. 2) Notably, adopting a softer dynamic strategy, such as the logarithmic scheduler, significantly improves learning performance. It's attributed to the logarithmic scheduler's ability to circumvent hard threshold-like indicator functions, which, in turn, provide a continuous and differentiable optimization objective for the student network. Such a continuous setting enables more efficient discovery of global minimum solutions.

\begin{table}[!htbp]
    \centering
    \caption{Ablation study of $\alpha_{adv}(\tau)$. $\infty$ represents we do not add $\mathcal{L}_{\theta}(x)$ in Equ. \ref{equ: dfkd}. T: ResNet34, S: ResNet18, Benchmark: CIFAR10.}
    \label{table: adv}
    \begin{tabular}{c|ccc}
    \hline
    $\alpha_{adv}(\tau)$                         & {Acc@1} & { Agree@1} & $L_p$  \\ \hline
    0                                         & 93.63                        & 95.98                          & 0.8462 \\
    $\infty$                                  & 94.92                        & 97.64                          & 0.8823 \\
    Equ \ref{equ: grad_adv} & \textbf{95.28}                        & \textbf{98.20}                          & \textbf{0.8915} \\ \hline
    \end{tabular}
    \end{table}

\begin{figure}[!htbp]
    \centering
    \begin{subfigure}[b]{0.25\textwidth}
        \centering
        \includegraphics[width=\textwidth]{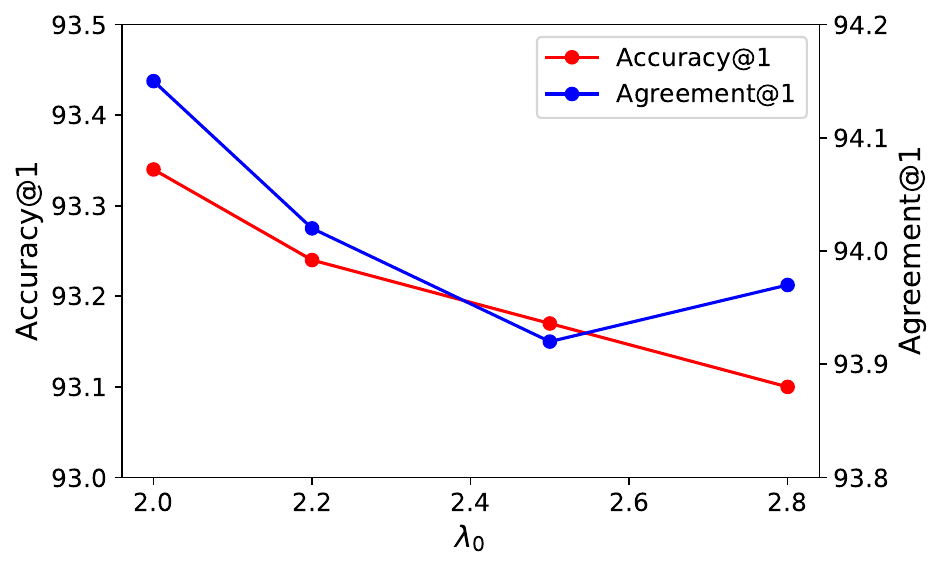}
        \caption{$Acc@1$ and $Agree@1$ with different $\lambda_0$s. }
        \label{fig:strategy_metricc}
    \end{subfigure}
    \hfill
    \begin{subfigure}[b]{0.225\textwidth}
        \centering
        \includegraphics[width=\textwidth]{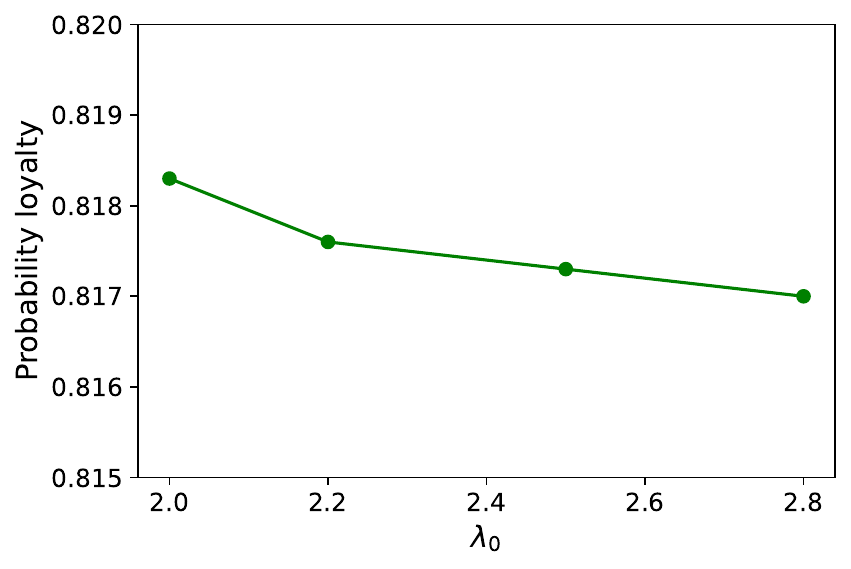}
        \caption{$L_p$ with different settings of $\lambda_0$.}
        \label{fig:strategy_metricd}
    \end{subfigure}
       \caption{The performance of CuDFKD with different settings of $\lambda_0$. Teacher: Noisy ResNet34, Student: ResNet18, Benchmark: CIFAR10. Better viewed in color.}
       \label{fig:lambda_0}
\end{figure}

\textbf{Effect of $\lambda_0$.} We investigate the effect of hyperparameter $\lambda(\tau)$ on the actual dynamic strategy here as discussed in Section \ref{sec: theo}. Choosing a proper $\lambda_0 = \lambda(0)$ is crucial as it affects the initial reweighting of generated pseudo-samples and early student training. Specifically, a large $\lambda_0$ leads to overreweighting, where almost all the generated samples are included in the training. Besides, the SP-regularizer $g(\mathbf{v}; \lambda)$ is close to zero, resulting in a collapse to the same training process as previous DFKD methods. On the other hand, a small $\lambda_0$ leads to underreweighting, where almost no samples are included in the training. Most of the contribution to the objective function $\mathbf{F}$ is from $g(\mathbf{v}; \lambda)$, causing no optimization of the student model.

Practically, we set $\lambda_0$ from $\{2.0, 2.2, 2.5, 2.8\}$, as choosing a value less than 2.0 results in underreweighting and divergence, while a value greater than 3.0 results in overreweighting. The experiments use the logarithm strategy for $\mathbf{v}^*$, with a noisy ResNet34 as the teacher and ResNet18 as the student. The noisy teacher model is used because we solely explore the hyperparameter sensitivity of the dynamic module. The results in Figure \ref{fig:lambda_0} indicate that $\lambda_0$ is slightly negatively correlated with the final performances. Besides, all the performances achieve more than 93\% $Acc@1$, which is better than other SOTA methods. For the $Agree@1$ and probability loyalty $L_p$ metrics, a similar phenomenon is observed as $Acc@1$. Obviously, using a proper interval of $\lambda_0$ leads to ideal performances for DFKD.

\begin{figure*}[!htbp]
    \centering
    \begin{subfigure}[b]{0.31 \textwidth}
        \centering
        \includegraphics[width=\textwidth]{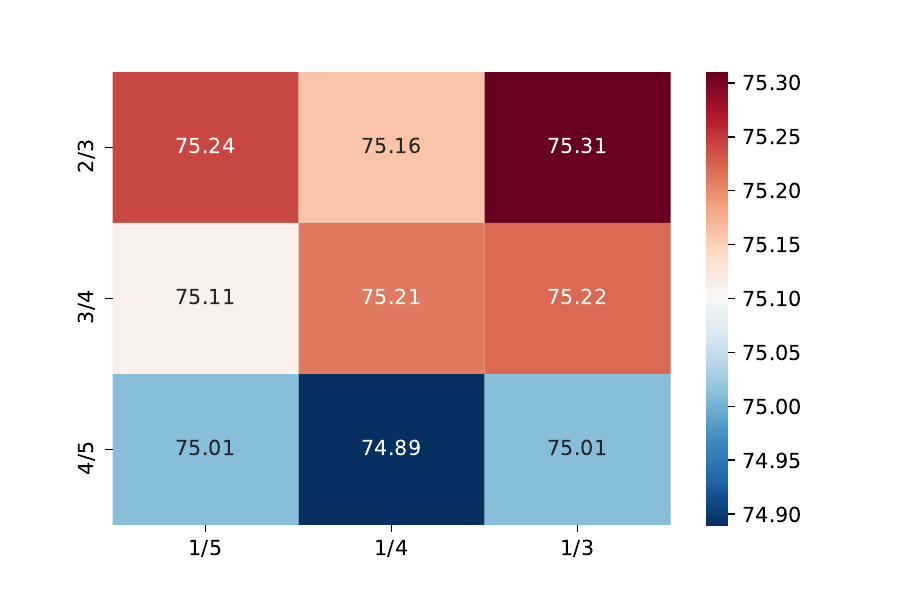}
        \caption{Acc@1, Res34 $\rightarrow$ Res18.}
        \label{fig:my_label1}
    \end{subfigure}
    \begin{subfigure}[b]{0.31 \textwidth}
        \centering
        \includegraphics[width=\textwidth]{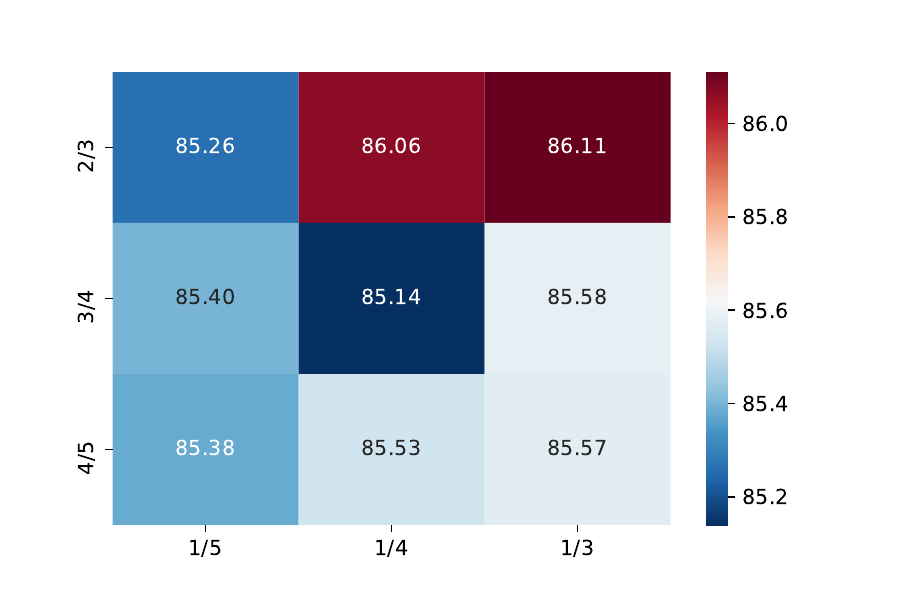}
        \caption{Agree@1, Res34 $\rightarrow$ Res18.}
        \label{fig:my_label2}
    \end{subfigure}
    \begin{subfigure}[b]{0.31 \textwidth}
        \centering
        \includegraphics[width=\textwidth]{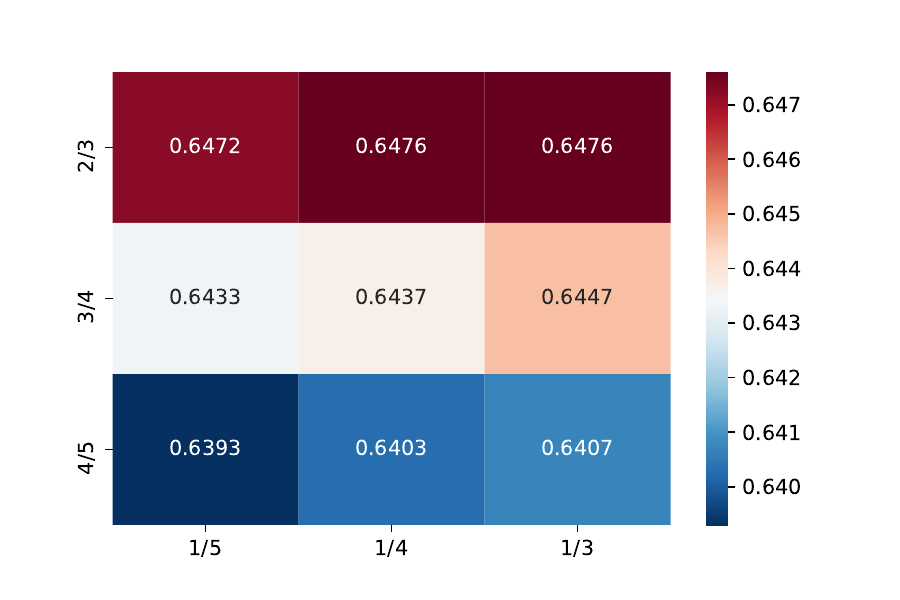}
        \caption{$L_p$, Res34 $\rightarrow$ Res18.}
        \label{fig:my_label3}
    \end{subfigure}
    \\
    \begin{subfigure}[b]{0.31 \textwidth}
        \centering
        \includegraphics[width=\textwidth]{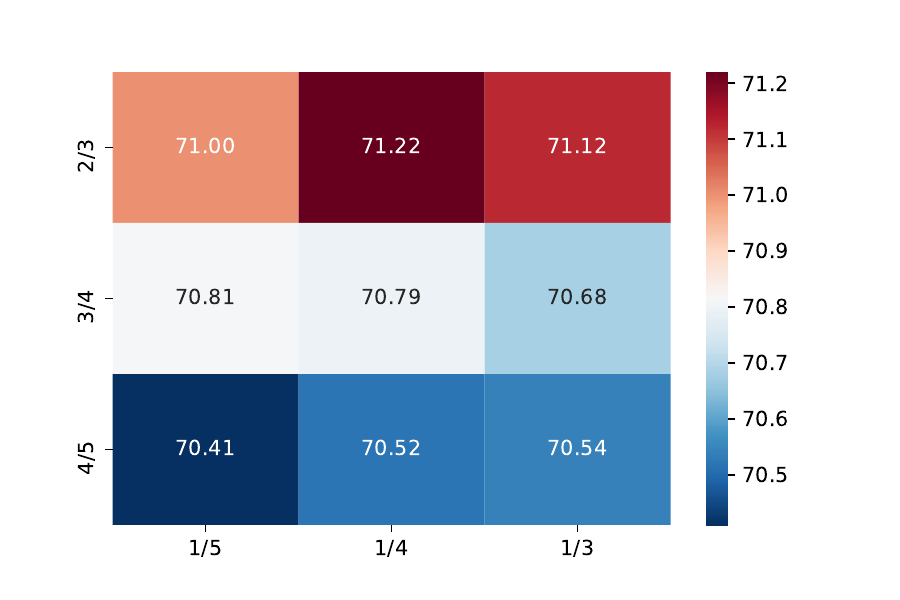}
        \caption{Acc@1, VGG11 $\rightarrow$ Res18.}
        \label{fig:my_label4}
    \end{subfigure}
    \begin{subfigure}[b]{0.31 \textwidth}
        \centering
        \includegraphics[width=\textwidth]{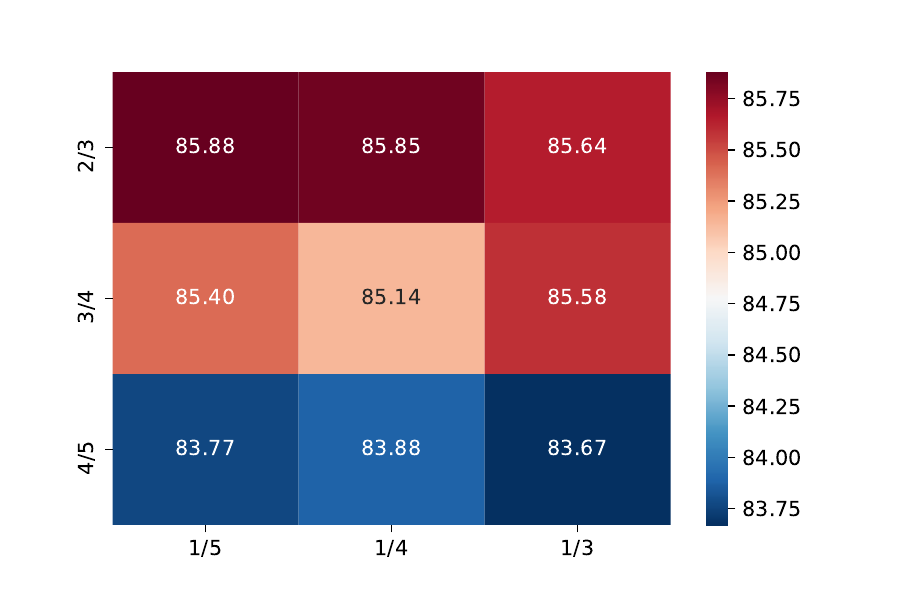}
        \caption{Agree@1, VGG11 $\rightarrow$ Res18.}
        \label{fig:my_label5}
    \end{subfigure}
    \begin{subfigure}[b]{0.31 \textwidth}
        \centering
        \includegraphics[width=\textwidth]{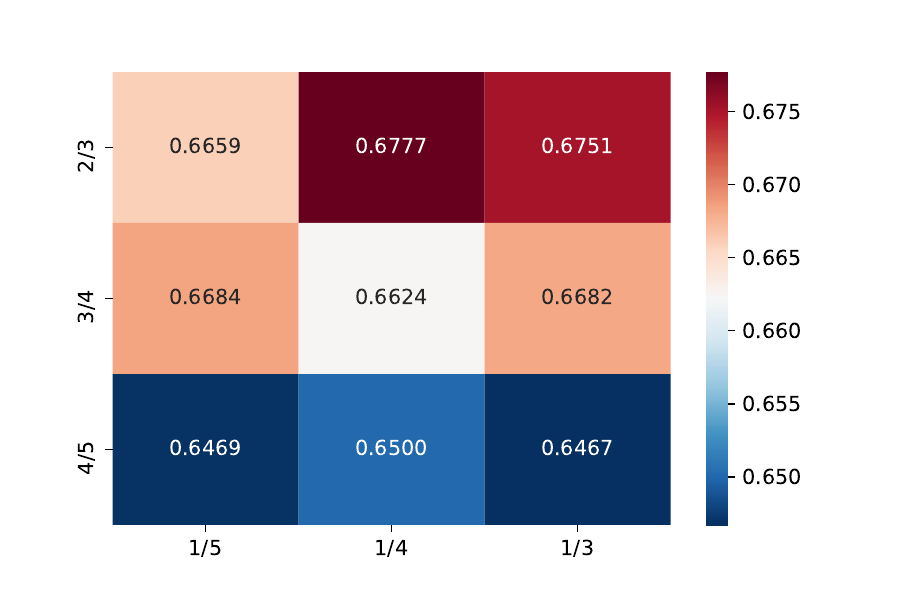}
        \caption{$L_p$, VGG11 $\rightarrow$ Res18.}
        \label{fig:my_label6}
    \end{subfigure}
    \\
    \begin{subfigure}[b]{0.31 \textwidth}
        \centering
        \includegraphics[width=\textwidth]{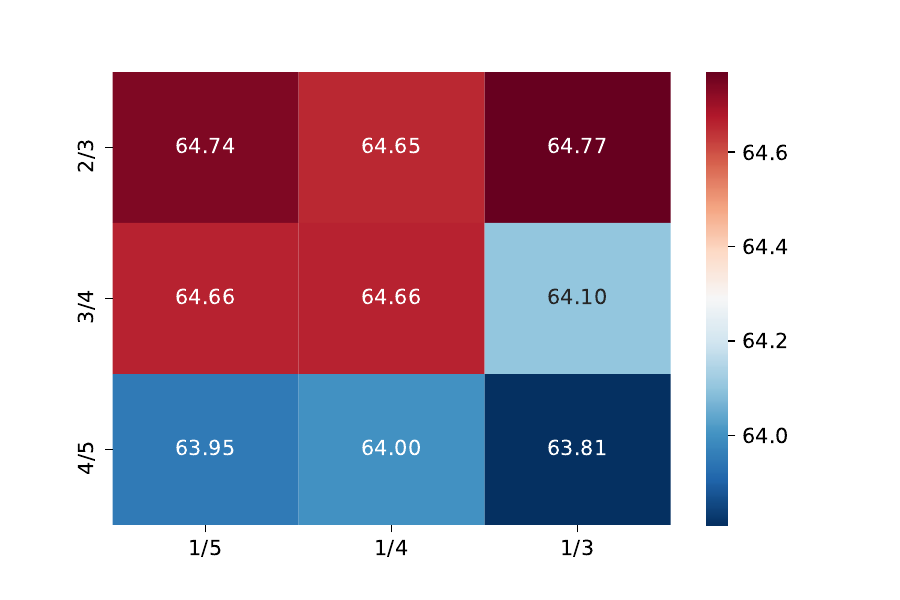}
        \caption{Acc@1, WRN-40-2 $\rightarrow$ WRN-16-2.}
        \label{fig:my_label7}
    \end{subfigure}
    \begin{subfigure}[b]{0.31 \textwidth}
        \centering
        \includegraphics[width=\textwidth]{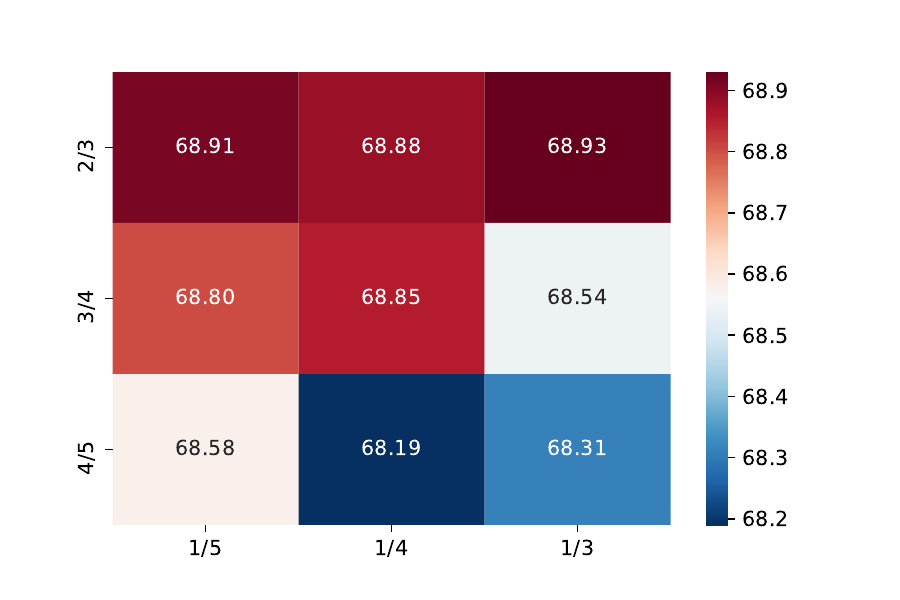}
        \caption{Agree@1, WRN-40-2 $\rightarrow$ WRN-16-2.}
        \label{fig:my_label8}
    \end{subfigure}
    \begin{subfigure}[b]{0.31 \textwidth}
        \centering
        \includegraphics[width=\textwidth]{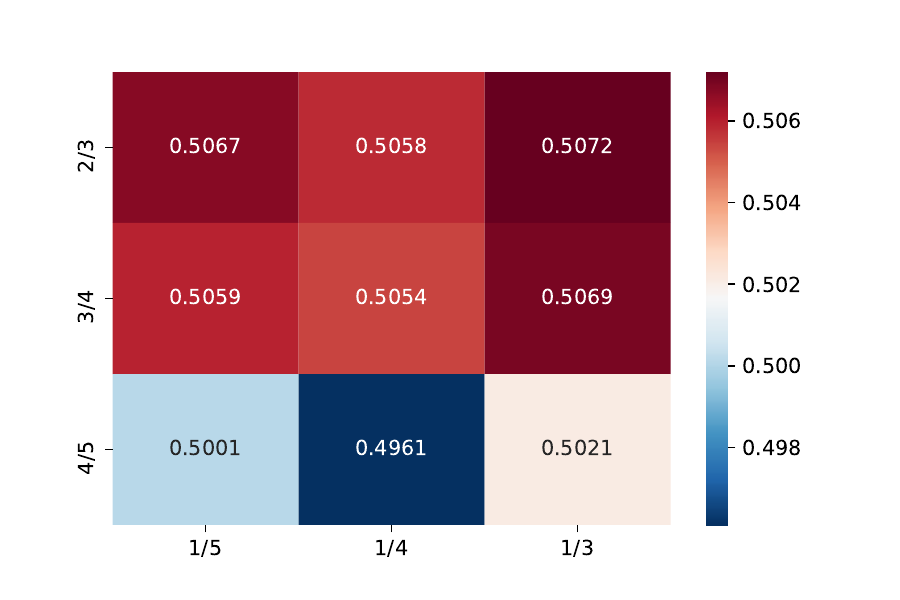}
        \caption{$L_p$, WRN-40-2 $\rightarrow$ WRN-16-2.}
        \label{fig:my_label9}
    \end{subfigure}
    \caption{The effect of $k_{begin}$(x-axis) and $k_{end}$(y-axis) on the performance of DFKD. All the experiments are performed on CIFAR100. The warmer the color is, the higher the metric is and thus the better performance is. Better viewed in color.}
    \label{fig: grad_adv}
        
    \end{figure*}


\textbf{Effect of dynamic strategy of generation stage, $\alpha_{adv}$.} The preceding section deals with the training stage, which involves optimizing $\mathbf{\phi}_s$. To investigate the hyperparameter sensitivity in the generation stage, we focus on the hyperparameter $\alpha_{adv}$, which regulates the gradient for the divergence to the weight of generator $\theta$. $\alpha_{adv}$ plays a crucial role in dynamically adjusting the difficulty level during the generation stage, as outlined in Section \ref{sec: method}. Our experiments are conducted on the CIFAR100 benchmark using VGG11-ResNet18, ResNet34-ResNet18, and WRN-40-2-WRN-40-1 pairs. We set the batch size to 768 to obtain better results and employed a simple scheduler for the experiments.

\begin{equation} \label{equ: grad_adv}
    \alpha_{adv}(\tau)=\left\{\begin{aligned}
0  &, \quad  \tau \leq k_{begin} N \\
\alpha \cdot \tau &, \quad  k_{begin} N < \tau \leq  k_{end} N \\
\lambda_{adv, final} &, \quad  \tau >  k_{end} N,
\end{aligned}\right.
\end{equation}

In previous work \cite{binici2022robust}, epoch 0 to $k_{begin} N$ are referred to as the warm-up stage. Practically, we set $k_{end} < 1$ to avoid extra noise to the pseudo samples. We search for hyperparameters by adjusting $k_{begin} \in \{1/5, 1/4, 1/3\}$ and $k_{end} \in \{2/3, 3/4, 4/5\}$, where $N$ is the total number of training epochs. The results are presented in Figure \ref{fig: grad_adv}. The figure shows that $k_{begin}$ is more sensitive than $k_{end}$, and all the best performances are obtained when $k_{begin} = 1/5$. In lighter networks like the wider ResNet series, the change in $Acc@1$ is more pronounced than in $Agree@1$ and $L_p$. It suggests that the training epochs can be treated as warm-up steps in previous DFKD work \cite{choi2020data,binici2022robust}, and it provides a prior for tuning the gradient. To test the effect of $\alpha_{adv}(\tau)$, we evaluated CuDFKD with two extreme cases: $\alpha_{adv}(\tau) = 0$ and $\alpha_{adv}(\tau) = \infty$. The results are presented in Table \ref{table: adv}. We observe that when setting the generation target as the decision boundary, the performance can be greatly improved (line 2). The dynamic strategy also leads to significant improvement in the student model's performance (line 3).

\begin{figure}[!htbp]
    \centering
    \begin{subfigure}[b]{0.25\textwidth}
        \centering
        \includegraphics[width=\textwidth]{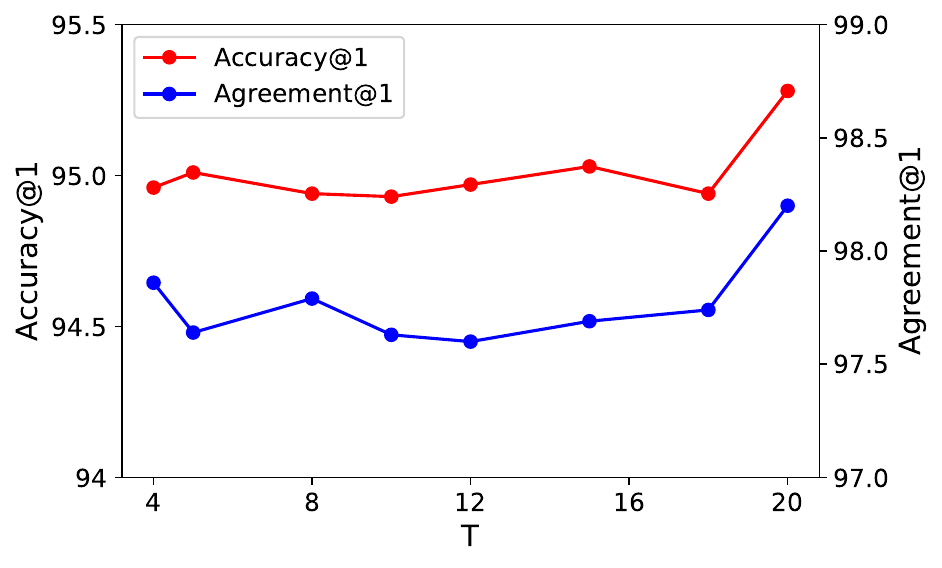}
        \caption{$Acc@1$ and $Agree@1$ with different $T$s.}
        \label{fig:strategy_metrica}
    \end{subfigure}
    \hfill
    \begin{subfigure}[b]{0.225\textwidth}
        \centering
        \includegraphics[width=\textwidth]{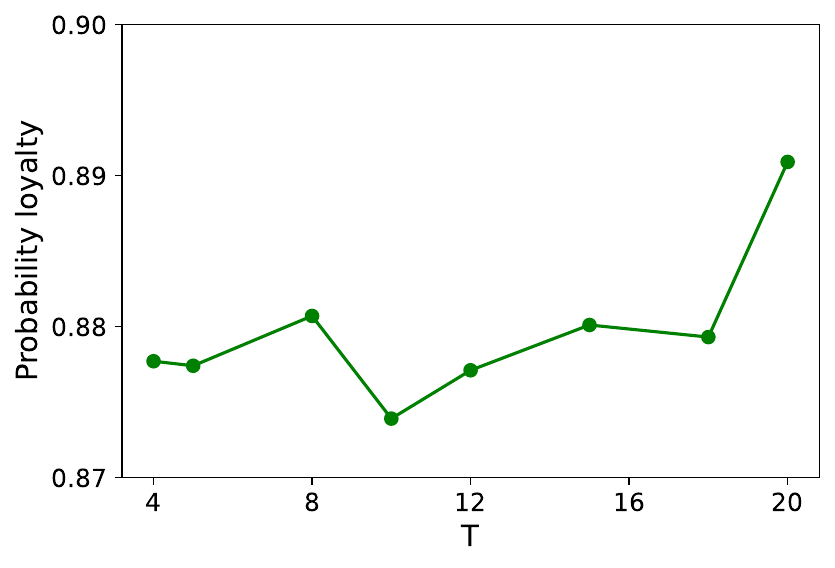}
        \caption{$L_p$ with different settings of $T$.}
        \label{fig:strategy_metricb}
    \end{subfigure}
       \caption{The performance of CuDFKD with different settings of temperature $T$. Teacher: ResNet34, Student: ResNet18, Benchmark: CIFAR10. Better viewed in color.}
       \label{fig:T}
\end{figure}

\textbf{Gradient of adversarial item $\alpha$}. The gradient parameter $\alpha$ in Equ. \ref{equ: grad_adv} plays a important role in controlling the variation of difficulty during generation. It significantly affects the difficulty scheduler in CuDFKD. To explore its sensitivity, we conduct experiments on CIFAR100 with $\alpha$ values in the range of [0.05, 0.4] using a step size of 0.05. The results are presented in Figure \ref{fig:alpha}. As $\alpha$ increases, all three metrics follow a similar pattern of initially increasing and decreasing later. When $\alpha$ is too small, the change in difficulty setting of CuDFKD is insufficient to produce a notable effect, resulting in a performance close to that of static DFKD. Conversely, when $\alpha$ is too large, the generated samples tend to focus on samples near the decision boundary in the later stages of training, making it more challenging to capture image information for classification.

\textbf{Temperature $T$ for KD.} Previous KD frameworks used the temperature $T$ to smooth the output distribution of the model and facilitate alignment between teacher and student models across the distribution. However, we observed that different model architectures require different $T$s for optimal results. To investigate this further, we conducted several comparison experiments by varying the temperature $T$ and comparing the resulting performance across three metrics. The results are presented in Figure \ref{fig:T}. Our findings indicate that for ResNet and VGG architectures, larger temperature values than those commonly used in data-driven KD (i.e., $T=4$\cite{hinton2015distilling}) yield better results. Additionally, heavier models require smoother output distributions to achieve optimal KD performance.

\section{Conclusion}\label{sec: future}
This paper proposes a novel dynamic Data-Free Knowledge Distillation (DFKD). Specifically, the dynamic module is realized by the thought of self-paced learning(SPL). In contrast to prior DFKD methods, CuDFKD utilizes dynamic difficulty setting, with the timestamp of the training process as target information for generating pseudo samples. By applying SPL, we dynamically adjust the difficulty of the generated samples, which improves the balance between the data-prior and decision boundary and makes the DFKD model more adaptive to the student model. We also analyze the properties of CuDFKD and demonstrate its ability to converge all parameters due to MM theory.

In our experiment, we compare the $Acc@1$, $Agree@1$ and $L_p$ of CuDFKD with several memory-based methods ,and present that CuDFKD achieves comparable results. Moreover, CuDFKD's convergence is fast and stable, and it performs the best across a wide range of hyperparameter settings. Overall, we offer a brand new perspective on how to perform DFKD dynamically and highlight the types of knowledge that can improve the training of KD. CuDFKD is a promising approach that provides an effective solution to the challenges posed by previous DFKD methods and has the potential to be applied in various settings.

\section{Limitation and Future work}

Despite its comparable results, CuDFKD has some limitations that leaves for future work. Firstly, the scheduler $\alpha_{adv}(\tau)$ design is currently manual, as we choose the best validation accuracy among all experiments. We plan to develop more automatic strategies to optimize this parameter in the future. Secondly, the curriculum setting of matching $\mathbf{v}^*(\lambda, \mathbf{L})$ and $d(p_{\theta}(x), \tau)$ can be made more flexible to achieve better performance. Additionally, we aim to explore the integration of sequential deep generative models with CuDFKD to control the difficulty of generated samples at different timestamps. These avenues of research will improve the applicability and effectiveness of CuDFKD.

\section{Acknowledgement}
This work is supported by Alibaba-Zhejiang University Joint Institute of Frontier Technologies, The National Key R\&D Program of China (No. 2021YFB2701100), the National Natural Science Foundation of China (No. 61972349,  62106221)  and the Fundamental Research Funds for the Central Universities(No. 226-2022-00064).

\printcredits

\bibliographystyle{cas-model2-names}

\bibliography{cas-refs}

\section{Appendix}

\subsection{Implementation Details}

In our details for implementation, we design a generator to represent the deep generative model, i.e., $x = G(z)$. The detailed hyperparameter design is in Table \ref{table: hyper}. We use Pytorch and code forked by CMI\cite{fang2021contrastive}. Code can be found in the GitHub link in section \ref{sec: exp}. Noticing that the hyperparameter selection in Table \ref{table: hyper} is the default set of hyperparameters in CuDFKD, we change one of them to some specific value when performing hyperparameter searching in section \ref{subsec: hyper}. 

In this paper's table values, we test different sets of nearby hyperparameters in Table \ref{table: hyper} and report the best value in the table. 

\begin{table}[!h]
    \centering
    \caption{Hyperparameter setting in the experiment.}
    \label{table: hyper}
    \begin{tabular}{cccc}
    \hline
    DataSet                & CIFAR10 & CIFAR100 & Tiny ImageNet \\ \hline
    Batch Size             & 768     & 768      & 512           \\
    Dim of $z$             & 512     & 512      & 512           \\
    $lr_G$                 & 0.001   & 0.001    & 0.001         \\
    $lr_s$                 & 0.1     & 0.1      & 0.1           \\
    $\alpha_{bn}$          & 1       & 1        & 1             \\
    $\alpha_{oh}$          & 20      & 20       & 20            \\
    $k_{begin}$            & 0.25    & 0.25     & 0.33          \\
    $k_{end}$              & 0.75    & 0.75     & 0.67          \\
    $\alpha$               & 1e-4    & 2e-5     & 2e-5          \\
    $\frac{\partial \lambda(\tau)}{\partial \tau}$ & 0.5     & 0.2      & 0.2           \\
    $\lambda_0$            & 2.0     & 1.0      & 1.0           \\
    loss type              & kl      & $l_1$       & $l_1$            \\
    $T$                    & 20      & 5        & 5             \\ \hline
    \end{tabular}
    \end{table}

\textbf{Implementation on CIFARs.} For the experiments of CIFARs, we design a generator motivated by DCGAN and other work by 2 upsampling layers. The structure is: Linear - Upsampling2 - Conv3 - BN - LeakyReLU - Upsampling2 - Conv3 - BN - LeakyReLU - Conv3 - Tanh.

\textbf{Implementation on Tiny ImageNet.} For the experiments of CIFARs, we design a generator motivated by DCGAN and other work by three upsampling layers. The structure is: Linear - Upsampling2 - Conv3 - BN - LeakyReLU - Upsampling2 - Conv3 - BN - LeakyReLU - Conv3 - BN - LeakyReLU - Upsampling2- Conv3 - Tanh. 

In the experiments, we find that the number of upsampling layers $L$ greatly affects the final performance of DFKD. When $L$ is too large, the interpolation operation can greatly affect the generation quality. On the other hand, if $L$ is too small, the hidden units for the linear projection will be large and thus lose a lot of information. In CuDFKD, we set $L=2,3$ for a better generation of pseudo samples.

For the implementation of previous DFKD methods, we use the single generator in the setting of DFQ and CMI, without any meta-data setting described in \cite{luo2020large}.


\end{document}